\documentclass{article} %
\usepackage{iclr2026_preprint,times}
\iclrfinalcopy

\usepackage[pagebackref]{hyperref}
\usepackage{url}

\title{Contrastive Predictive Coding Done Right \\for Mutual Information Estimation}

\author{
J. Jon Ryu\textsuperscript{1}, 
Pavan Yeddanapudi\textsuperscript{1}, 
Xiangxiang Xu\textsuperscript{2}, 
Gregory W.~Wornell\textsuperscript{1} \\
\textsuperscript{1}Department of EECS, MIT, Cambridge, MA 02139, USA \\
\textsuperscript{2}Department of Computer Science, University of Rochester, Rochester, NY 14627, USA \\
\texttt{\{jongha,pky,gww\}@mit.edu}, \texttt{xiangxiangxu@rochester.edu}
}

\usepackage{xspace}
\usepackage{bbm}
\input{mathlig}
\usepackage{mathtools}
\usepackage{relsize}
\usepackage{mathrsfs}
\usepackage{dsfont}
\DeclarePairedDelimiterX{\inp}[2]{\langle}{\rangle}{#1, #2}
\makeatletter
\newcommand*\bigcdot{\mathpalette\bigcdot@{.5}}
\newcommand*\bigcdot@[2]{\mathbin{\vcenter{\hbox{\scalebox{#2}{$\m@th#1\bullet$}}}}}
\makeatother

\newcommand{\muspace}{\mspace{1mu}}

\DeclareRobustCommand{\scond}{\mathchoice{\muspace\vert\muspace}{\vert}{\vert}{\vert}}
\mathlig{|}{\scond}

\DeclareRobustCommand{\discint}{\mathchoice{\mspace{-1.5mu}:\mspace{-1.5mu}}{\mspace{-1.5mu}:\mspace{-1.5mu}}{:}{:}}
\mathlig{::}{\discint}
\newcommand{\suchthat}{\mathchoice{\colon}{\colon}{:\mspace{1mu}}{:}}

\newcommand{\Ac}{\mathcal{A}}

\newcommand{\Dc}{\mathcal{D}}

\newcommand{\Lc}{\mathcal{L}}

\newcommand{\Xc}{\mathcal{X}}
\newcommand{\Yc}{\mathcal{Y}}
\newcommand{\Zc}{\mathcal{Z}}

\newcommand{\fv}{{\bf f}}
\newcommand{\gv}{{\bf g}}

\newcommand{\eb}{{\mathbf e}}

\newcommand{\rb}{{\mathbf r}}

\newcommand{\rhob}{\boldsymbol{\rho}}

\def\a{\alpha}
\def\b{\beta}

\def\th{\theta}

\DeclareMathOperator\E{\mathsf{E}}
\let\P\relax
\DeclareMathOperator\P{\mathsf{P}}

\newcommand\eg{e.g.,\xspace}
\newcommand\ie{i.e.,\xspace}
\def\textiid{i.i.d.\@\xspace}
\newcommand\iid{\ifmmode\text{ i.i.d. } \else \textiid \fi}

\newcommand{\Real}{\mathbb{R}}

\newcommand{\ones}{\mathds{1}}
\newcommand{\half}{\frac{1}{2}}%

\def\mathllap{\mathpalette\mathllapinternal}
\def\mathllapinternal#1#2{%
  \llap{$\mathsurround=0pt#1{#2}$}}

\def\clap#1{\hbox to 0pt{\hss#1\hss}}
\def\mathclap{\mathpalette\mathclapinternal}
\def\mathclapinternal#1#2{%
  \clap{$\mathsurround=0pt#1{#2}$}}

\let\oldstackrel\stackrel
\renewcommand{\stackrel}[2]{\oldstackrel{\mathclap{#1}}{#2}}

\DeclarePairedDelimiterX{\infdivx}[2]{(}{)}{%
  #1\;\delimsize\|\;#2%
}

\renewcommand{\hbar}{h\mathllap{\overline{\vphantom{h}\hphantom{\rule{4.6pt}{0pt}}}\mspace{0.77mu}}}

\catcode`~=11 %
\newcommand{\urltilde}{\kern -.06em\lower -.06em\hbox{~}\kern .02em}
\catcode`~=13 %

\DeclarePairedDelimiterX{\norm}[1]{\lVert}{\rVert}{#1}
\DeclarePairedDelimiterX{\abs}[1]{\lvert}{\rvert}{#1}

\usepackage{xparse}

\let\oldpartial\partial
\renewcommand*{\partial}{\mathop{}\!\oldpartial}
\newcommand{\defeq}{\mathrel{\mathop{:}}=}

\usepackage{makecell}
\usepackage{tablefootnote}
\usepackage{pdflscape}
\usepackage{booktabs}

\usepackage{eufrak}
\DeclareMathAlphabet{\mathpzc}{OT1}{pzc}{m}{it}

\usepackage{amsmath}
\usepackage{amssymb}
\usepackage{mathtools}
\usepackage{amsthm}

\usepackage{scrwfile}
\TOCclone[\appendixname]{toc}{atoc}
\newcommand\StartAppendixEntries{}
\AfterTOCHead[toc]{%
  \renewcommand\StartAppendixEntries{\value{tocdepth}=-10000\relax}%
}
\AfterTOCHead[atoc]{%
  \edef\maintocdepth{\the\value{tocdepth}}%
  \value{tocdepth}=-10000\relax%
  \renewcommand\StartAppendixEntries{\value{tocdepth}=\maintocdepth\relax}%
}
\usepackage[normalem]{ulem}

\newcommand{\jon}[1]{\textcolor{red}{\bf Jon: #1}}

\usepackage{adjustbox}

\usepackage{amsthm}
\newtheorem{theorem}{Theorem}

\newtheorem{lemma}[theorem]{Lemma}
\newtheorem{corollary}[theorem]{Corollary}

\newtheorem{proposition}[theorem]{Proposition}

\theoremstyle{definition}
\newtheorem{definition}[theorem]{Definition}
\newtheorem{remark}[theorem]{Remark}

\renewcommand{\E}{\mathbb{E}}

\usepackage{listings}
\usepackage{pifont}%

\usepackage{transparent}
\definecolor{codegreen}{rgb}{0,0.6,0}
\definecolor{codegray}{rgb}{0.5,0.5,0.5}
\definecolor{codepurple}{rgb}{0.58,0,0.82}
\definecolor{backcolour}{rgb}{0.95,0.95,0.92}

\lstdefinestyle{mystyle}{
    language=Python,
    keywords={def,class,return,None,raise,mean,sqrt},    
    backgroundcolor=\color{backcolour},   
    commentstyle=\color{codegreen},
    keywordstyle=\color{blue},
    numberstyle=\tiny\color{codegray},
    stringstyle=\color{codepurple},
    basicstyle=\ttfamily\footnotesize,
    breakatwhitespace=false,         
    breaklines=true,                 
    captionpos=b,                    
    keepspaces=true,                 
    numbers=left,                    
    xleftmargin=.375cm,
    numbersep=5pt,
    showspaces=false,                
    showstringspaces=false,
    showtabs=false,                  
    tabsize=4,    
}
\usepackage{minted}
\usepackage[scale=0.72]{FiraMono}

\newif\ifFINAL

\FINALtrue %

\ifFINAL
  \def\jon#1{}
\else
  \usepackage{transparent}
  \usepackage[inline]{showlabels}
  \renewcommand{\showlabelfont}%
  {\transparent{0.8}\scriptsize\bf\slshape\color{Lavender}}
\fi

\usepackage{subcaption}

\newcommand{\pdata}{{\textcolor{orange}{q_1}}}
\newcommand{\pnoise}{{\textcolor{blue}{q_0}}}

\newcommand{\numberthis}{\addtocounter{equation}{1}\tag{\theequation}}

\newcommand{\etabref}{\etab^*}
\newcommand{\etaref}{\eta^*}
\newcommand{\rhobref}{\rhob^*}

\newcommand{\fbar}{\overline{f}}

\newcommand{\gb}{\mathbf{g}}

\newcommand{\pib}{\boldsymbol{\pi}}
\newcommand{\zerob}{\boldsymbol{0}}
\newcommand{\Reg}{\mathsf{Reg}}

\newcommand{\lambdab}{\boldsymbol{\lambda}}

\newcommand{\rv}{\mathbf{r}}

\newcommand{\etab}{\boldsymbol{\eta}}
\newcommand{\etabdata}{\etab}

\renewcommand{\defeq}{\triangleq}
\usepackage{arydshln}
\usepackage[utf8]{inputenc} %
\usepackage[T1]{fontenc}    %
\usepackage{url}            %
\usepackage{booktabs}       %
\usepackage{amsfonts}       %
\usepackage{nicefrac}       %
\usepackage{microtype}      %
\usepackage[dvipsnames]{xcolor}

\usepackage{enumitem}
\usepackage{mleftright}

\usepackage{multirow}
\usepackage[normalem]{ulem}

\usepackage{comment}

\begin{document}

\maketitle

\begin{abstract}
The InfoNCE objective, originally introduced for contrastive representation learning, has become a popular choice for mutual information (MI) estimation, despite its indirect connection to MI. In this paper, we demonstrate why InfoNCE should not be regarded as a valid MI estimator, and we introduce a simple modification, which we refer to as \emph{InfoNCE- anchor}, for accurate MI estimation. 
Our modification introduces an auxiliary \emph{anchor} class, enabling consistent density ratio estimation and yielding a plug-in MI estimator with significantly reduced bias.
Beyond this, we generalize our framework using proper scoring rules, which recover InfoNCE-anchor as a special case when the log score is employed. 
This formulation unifies a broad spectrum of contrastive objectives, including NCE, InfoNCE, and $f$-divergence variants, under a single principled framework. 
Empirically, we find that InfoNCE-anchor with the log score achieves the most accurate MI estimates; however, in self-supervised representation learning experiments, we find that the anchor does not improve the downstream task performance.
These findings corroborate that contrastive representation learning benefits not from accurate MI estimation per se, but from the learning of structured density ratios.
\end{abstract}

\section{Introduction}
Contrastive learning has become a cornerstone of modern unsupervised representation learning, powering advances in computer vision~\citep{Chen--Kornblith--Norouzi--Hinton2020}, natural language processing~\citep{Mikolov--Chen--Corrado--Dean2013,Levy--Goldberg2014}, and beyond~\citep{Jaiswal--Babu--Zadeh--Banerjee--Makedon2020,Hu--Wang--Zhang--Chen--Guan2024}. One of the key ingredients in many contrastive methods is the \emph{InfoNCE} objective~\citep{van-den-Oord--Li--Vinyals2018}. While originally proposed as a representation learning framework, \citet{van-den-Oord--Li--Vinyals2018} noted that the InfoNCE objective can be used to evaluate mutual information (MI), interpreting the objective as a variational bound on MI; see \citep[Appendix A]{van-den-Oord--Li--Vinyals2018}. This interpretation has led to its widespread use for MI estimation, e.g., \citep{Poole--Ozair--van-den-oord-Alemi--Tucker2019,Song--Ermon2020a,Gowri--Lun--Klein--Yin2024,Lee--Rhee2024}.
It is also widely known, however, that InfoNCE often yields a rather loose bound on MI~\citep{Poole--Ozair--van-den-oord-Alemi--Tucker2019,Tschannen--Djolonga--Rubenstein--Gelly--Lucic2020}. As a result, the InfoNCE estimator is generally considered a low-variance but high-bias MI estimator~\citep{van-den-Oord--Li--Vinyals2018}. Although several proposals have been made to address this issue since its inception, the effectiveness (i.e., the low-variance property) and the limitation (i.e., the high bias) of InfoNCE remain poorly understood.

In this paper, we clarify the operational meaning of the InfoNCE objective, by showing that the objective should be understood as a variational lower bound of a statistical divergence different from the mutual information.
Building on this clarification, we establish a sharp characterization of its relationship to the Kullback--Leibler (KL) divergence, revealing why the InfoNCE objective should not be regarded as a direct estimate of MI.
We further argue that InfoNCE can be viewed as a density ratio estimation objective, while the critic (or its exponentiated form) estimates the density ratio $\frac{p(x,y)}{p(x)p(y)}$ only up to an arbitrary function $C(y)$, rendering it unsuitable for use in a plug-in estimator. %

To address this limitation, we introduce a simple modification to the variational objective, which we call \emph{InfoNCE-anchor}, corresponding to an alternative divergence.
In the new framework, the inclusion of an \emph{anchor} enables the (exponentiated) critic to estimate the density ratio $\frac{p(x,y)}{p(x)p(y)}$ directly.
This adjustment facilitates consistent density ratio estimation and yields a plug-in MI estimator that retains the low variance of InfoNCE while significantly reducing its bias.
See Figure~\ref{fig:mi_estimation} for a quick comparison with existing estimators, where the new plug-in estimator based on InfoNCE-anchor demonstrates low-bias, low-variance performance.

We generalize our framework using tools from statistical decision theory, showing that InfoNCE-anchor corresponds to the \emph{log score}, a canonical example of a proper scoring rule.  
This insight reveals that many contrastive objectives, including NCE, InfoNCE, and certain $f$-divergence variants, can be unified under a single principled framework of density ratio estimation using proper scoring rules.  

Empirically, we show that estimators induced by the log score yields state-of-the-art MI estimates across a range of settings.  
In contrastive representation learning tasks, however, we find that multiple scoring rules yield similar performance, suggesting that MI estimation is not the primary driver of contrastive learning success.  
Instead, our results support that contrastive learning benefits from learning structured density ratios, regardless of whether the objectives are accurate MI estimators.  

While InfoNCE has played a canonical role in contrastive representation learning, its rather loose association to MI estimation has historically fostered the misconception that representation learning is essentially about maximizing MI; see, e.g., \citep{Bachman--Hjelm--Buchwalter2019learning,Wu--Zhuang--Mosse--Yamins--Goodman2020}.  
This paper clarifies why such an interpretation can be limiting and imprecise, and that contrastive representation learning should instead be framed as representation learned to factorize pointwise MI (PMI) $\log \frac{p(x,y)}{p(x)p(y)}$, or pointwise dependence (PD) $\frac{p(x,y)}{p(x)p(y)}$.  
All proofs can be found in Appendix~\ref{app:sec:proofs}.
Logarithms in this paper are in base 2 and thus KL divergence and MI are in bit.

\section{Preliminaries}
In this paper, we first review different types of variational-bounds-based MI estimators in the literature, and provide a taxonomy.
We then delve into the InfoNCE estimator, and show why the InfoNCE estimator should not be considered as a direct estimate for MI.

\subsection{Types of Information Estimators}
\label{sec:taxonomy}

Existing variational-bound-based MI estimators can be categorized into three principal categories as follows, based on the relationship of their optimization objectives to the final metrics to compute MI. 
Table~\ref{tab:mi_estimators} summarizes the representative estimators.

\begin{itemize}
\item \textbf{Type 1: Training and evaluation with a single variational lower bound.} These estimators optimize a tractable lower bound on the MI and use the same bound for evaluation. Examples include DV~\citep{Donsker--Varadhan1975}, NWJ~\citep{Ngyuen--Wainwright--Jordan2010}, and InfoNCE~\citep{van-den-Oord--Li--Vinyals2018}. While conceptually simple and natural, 
\citet{McAllester--Stratos2020} showed that any distribution-free high-probability lower bound of MI is upper bounded by $\log N$, where $N$ is the sample size. This result implies that variational lower-bound–based sample estimates of MI suffer from an inherent limitation.

\item \textbf{Type 2: Training with a variational lower bound, evaluation by plugging-in to another variational lower bound.} 
These estimators optimize a surrogate objective, often smoothed or stabilized for optimization, and then estimate MI via plug-in to a different bound such as DV or NWJ. Examples include MINE~\citep{Belghazi--Baratin--Rajeshwar--Ozair--Bengio--Courville--Hjelm2018}, JS~\citep{Hjelm--Fedorov--Lavoie-Marchildon--Grewal--Bachman--Trischler--Bengio2018}, and SMILE~\citep{Song--Ermon2020a}. These methods often improve stability during training, but introduce additional sources of mismatch between optimization and evaluation.
Note that the critique of \citet{McAllester--Stratos2020} still applies to this type of estimators.

\item \textbf{Type 3: Training with a variational lower bound, evaluation with a plug-in estimator.} These estimators target to learn the density ratio $\frac{p(x,y)}{p(x)p(y)}$ directly and compute MI by plugging the estimated score function into the definition of MI. This includes recent methods like PCC/D-RFC~\citep{Tsai--Zhao--Yamada--Morency--Salakhutdinov2020} and $f$-DIME~\citep{Letizia--Novello--Tonello2024}, as well as our method to be proposed below. These approaches provide greater flexibility and potentially lower bias, side-stepping from the issue of the variational lower-bound approach, as they decouple density ratio learning from specific bounds.
\end{itemize}

\begin{table}[tbh]
\centering
\caption{Overview of existing variational-bound-based MI estimators. In this table, we use the standard critic parametrization, which aims to train $c(x,y)\approx \log\frac{p(x,y)}{p(x)p(y)}$. }%
\resizebox{\textwidth}{!}{%
\begin{tabular}{llll}
\toprule
& \textbf{Estimator} & \textbf{Optimization objective $ \Lc(c) $ (loss)} & \textbf{Estimator $ \hat{I}(X; Y) $} \\ \midrule
\parbox[t]{2mm}{\multirow{3}{*}{\rotatebox[origin=c]{90}{Type 1}}} 
& DV~\citep{Donsker--Varadhan1975} &
$ \Lc_{\mathrm{DV}}(c) 
\defeq -\E_{p(x, y)}[c(x, y)] + \log \E_{p(x)p(y)}[e^{c(x, y)}] $ &
$-\Lc_{\mathrm{DV}}(c)$ \\
& NWJ~\citep{Ngyuen--Wainwright--Jordan2010} &
$ 
\Lc_{\mathrm{NWJ}}(c) 
\defeq -\E_{p(x, y)}[c(x, y)] + \E_{p(x)p(y)}[e^{c(x, y)-1}] $ &
$-\Lc_{\mathrm{NWJ}}(c) $ \\
& \makecell[l]{InfoNCE~\citep{van-den-Oord--Li--Vinyals2018}\\~or NT-XEnt~\citep{Chen--Kornblith--Norouzi--Hinton2020}} &
$ 
\Lc_{\mathrm{InfoNCE}}(c) 
\defeq -\E_{p^K(x, y)} \Bigl[ \frac{1}{K} \sum_{i=1}^K \log \frac{c(x_i, y_i)}{\frac{1}{K} \sum_{j=1}^K c(x_i, y_j)} \Bigr] $ &
$-\Lc_{\mathrm{InfoNCE}}(c)$ \\
\midrule
\parbox[t]{2mm}{\multirow{3}{*}{\rotatebox[origin=c]{90}{Type 2}}} 
& 
MINE~\citep{Belghazi--Baratin--Rajeshwar--Ozair--Bengio--Courville--Hjelm2018} &
$ \Lc_{\mathrm{MINE}}(c) \defeq -\E_{p(x, y)}[c(x, y)] + \frac{\E_{p(x)p(y)}[e^{c(x, y)}]}{\text{EMA}(\E_{p(x)p(y)}[e^{c(x, y)}])} $ &
$ -\Lc_{\mathrm{DV}}(c) $ \\
& \makecell[l]{JS~\citep{Poole--Ozair--van-den-oord-Alemi--Tucker2019}\\
~or NT-Logistics~\citep{Chen--Kornblith--Norouzi--Hinton2020}} &
$ \Lc_{\mathrm{JS}}(c)\defeq \E_{p(x, y)}[\text{sp}(-c(x, y))] + \E_{p(x)p(y)}[\text{sp}(c(x, y))] $ &
$ -\Lc_{\mathrm{NWJ}}(c) $ \\
& SMILE~\citep{Song--Ermon2020a} &
$ \Lc_{\mathrm{JS}}(c) $ &
$ -\Lc_{\mathrm{clippedDV}}(c) $
\\
\midrule
\parbox[t]{2mm}{\multirow{3}{*}{\rotatebox[origin=c]{90}{Type 3}}} 
& PCC~/~D-RFC~\citep{Tsai--Zhao--Yamada--Morency--Salakhutdinov2020} & $\Lc_{\mathrm{JS}}(c)$~/~$\Lc_{\chi^2}(c)\defeq -2\E_{p(x, y)}[e^{c(x, y)})] + \E_{p(x)p(y)}[e^{2 c(x, y)}] $ & $\E_{\hat{p}(x,y)}[c(x,y)]$ \\
& $f$-DIME~\citep{Letizia--Novello--Tonello2024} & $\Lc_{f\textrm{-NWJ}}(c)$ & $\E_{\hat{p}(x,y)}[c(x,y)]$ \\
& InfoNCE-anchor & $\Lc_{K;\nu}^\Psi(c)$ (see Eq.~\eqref{eq:kway_obj} and Eq.~\eqref{eq:def_proposed}) & $\E_{\hat{p}(x,y)}[c(x,y)]$ \\
\bottomrule
\end{tabular}}
\label{tab:mi_estimators}
\end{table}

\subsection{Demystifying the InfoNCE Estimator}
Despite its inception as an objective for contrastive representation learning~\citep{van-den-Oord--Li--Vinyals2018}, InfoNCE has become widely considered as a MI estimator.
In this section, we revisit the analytical foundation of the objective and disentangle  what InfoNCE is \emph{claimed} to measure from what InfoNCE indeed characterizes.
Our goal is two‑fold:
(1) reveal the divergence that InfoNCE targets, and (ii) quantify the precise gap between that divergence and the mutual information.
Before we proceed, we remark that the core of InfoNCE can be better described when we contrast two abstract distributions $\pdata(x)$ and $\pnoise(x)$, which can be replaced by $p(x|y)$ and $p(x)$, respectively, if we wish to specialize it for mutual information.

Throughout, let $x_1$ denote a \emph{positive} example drawn from the data distribution $\pdata$, and let $x_2,\dots,x_K$ be \emph{negative} examples drawn i.i.d. from a noise distribution $\pnoise$.
We let $x_{i:j}\defeq (x_i,\dots,x_j)$ for $i\le j$ as a shorthand.
A score network (or critic) $r_\th\colon\Xc\to\Real_{>0}$ is trained to assign large values to real samples and small values to negatives, and the InfoNCE loss compares $r_\th(x_1)$ against the arithmetic mean of $r_\th(x_z)$ over the whole batch.
\begin{align*}
\Lc_{\text{InfoNCE}}(\th)
&\defeq-\Dc_{\text{InfoNCE}}(\th)
\defeq -\E_{\pdata(x_1)\pnoise(x_2)\cdots\pnoise(x_K)}\Biggl[\log\frac{r_\th(x_1)}{\frac{1}{K}\sum_{z=1}^K r_\th(x_z)}
\Biggr].
\end{align*}
As we alluded to earlier, if we plug-in $p(x|y)$ and $p(x)$ in place of $\pdata(x)$ and $\pnoise(x)$, respectively, then $\E_{p(y)}[\Lc_{\text{InfoNCE}}(\th)]$ recovers the standard InfoNCE objective for two modalities.

The following statement from \citep{van-den-Oord--Li--Vinyals2018,Poole--Ozair--van-den-oord-Alemi--Tucker2019} is a widely known connection between the InfoNCE objective to the KL divergence, which provides a justification of the InfoNCE objective as an MI estimator for $K$ sufficiently large.
We present its proof in Appendix~\ref{app:proof:prop:cpc_loose} for completeness. 
\begin{proposition}
\label{prop:cpc_loose}
$\Dc_{\text{InfoNCE}}(\th)
\le \min\{\log K, D(\pdata~\|~\pnoise)\}$.
\end{proposition}

Our first contribution is to provide a \emph{tight} upper bound on $\Dc_{\text{InfoNCE}}(\th)$, which yields a much sharper bound on $\Dc_{\text{InfoNCE}}(\th)$ than Proposition~\ref{prop:cpc_loose} as a corollary.
\begin{theorem}\label{thm:cpc}
For $z\in[K]$, define $p(x_{1:K}|z)$ as $p(x_{1:K}|z) 
\defeq \pdata(x_z) \prod_{i\neq z} \pnoise(x_i)$.
Then, we have
\begin{align*}
\Dc_{\text{InfoNCE}}(\th)
&\le D_{K\text{-JS}}(\pdata,\pnoise)
\defeq \frac{1}{K}\sum_{z=1}^K D\Bigl(p(x_{1:K}|z)~\Big\|~\frac{1}{K}\sum_{z'=1}^K p(x_{1:K}|z')\Bigr)\\
&\le \min\Bigl\{\log K,
D(\pdata~\|~\pnoise)-
\log \Bigl(\frac{1}{K} (2^{D(\pdata~\|~\pnoise)}-1)+1\Bigr)
\Bigr\}.
\end{align*}
The first inequality becomes equality if and only if $r_\th(x)\propto\frac{\pdata(x)}{\pnoise(x)}$.
\end{theorem}
This theorem establishes two key theoretical properties of the InfoNCE objective.
\textbf{First}, the InfoNCE objective is a \emph{tight} variational lower bound of $D_{K\text{-JS}}(\pdata,\pnoise)$, a generalization of Jensen--Shannon divergence (JSD) which we call the \emph{$K$-way JSD}. The InfoNCE objective becomes equal to the divergence $D_{K\text{-JS}}(\pdata,\pnoise)$ if and only if $r_\th(x)\propto \pdata(x)/\pnoise(x)$. 
Since it only learns the density ratio up to a multiplicative constant, one \emph{cannot} use it for a plug-in estimator (i.e., Type 3 in Section~\ref{sec:taxonomy}) with the critic (i.e., the density ratio model) learned by InfoNCE.
\textbf{Second}, the InfoNCE objective may be still far away from $D(\pdata~\|~\pnoise)$ even for $K$ such that $\log K\ge D(\pdata~\|~\pnoise)$.
Concretely, suppose $D(\pdata~\|~\pnoise)=2$. Then, the $\Dc_{\text{InfoNCE}}(\th)\le 1.19\ldots$ when $K=4$ even if $\log K\ge D(\pdata~\|~\pnoise)$, and even for $K=64$, we have $\Dc_{\text{InfoNCE}}(\th)\le 1.93\ldots$, which is strictly smaller than the KL divergence $D(\pdata~\|~\pnoise)=2$.
This clearly demonstrates that the InfoNCE objective $\Dc_{\text{InfoNCE}}(\th)$ can never match the KL divergence for any finite $K$ and hence is unsuitable as a direct surrogate for MI. 
This contrasts with other Type 1 estimators such as DV and NWJ, which provide \emph{tight} variational representations of the KL divergence; that is, their bound becomes equal to the KL divergence when the critic function is equal (or proportional) to the true log-density ratio.

In the next section, we propose a modification of the InfoNCE objective, such that the critic is learned to exactly estimate the density ratio $\frac{\pdata(x)}{\pnoise(x)}$, and so that it can be used in a plug-in estimator for density ratio functionals such as mutual information.

\section{Tensorized Density Ratio Estimation with Anchor}
Consider two distributions $\pnoise(x)$ and $\pdata(x)$.
To estimate the density ratio $\frac{\pdata(x)}{\pnoise(x)}$ using samples from $\pnoise(x)$ and $\pdata(x)$, we consider the following classification problem over $\Xc^K$ (hence \emph{tensorization}), where we define the class densities $p(x_{1:K}|z)$ for $z=0,1,\ldots,K$ as
\begin{align}
\begin{aligned}
\boxed{\text{class 0 (anchor)}:~\pnoise(x_1) \pnoise(x_2)\cdots \pnoise(x_K)}&\\
\text{class 1}:~\pdata(x_1) \pnoise(x_2)\cdots \pnoise(x_K)~&\\
\text{class 2}:~\pnoise(x_1) \pdata(x_2)\cdots \pnoise(x_K)~&\\
\vdots~~~~~~~~~~~~~~~~~~~\\
\text{class $K$}:~\pnoise(x_1) \pnoise(x_2)\cdots \pdata(x_K)~&
\end{aligned}
\label{eq:def_class}
\end{align}
and the prior probabilities over the classes as $p(z)=\frac{\nu}{K+\nu}$ for $z=0$, and $p(z)=\frac{1}{K+\nu}$ if $z\in[K]$,
for some $\nu\ge 0$. 
As highlighted, class 0 plays a special role as an \emph{anchor}, allowing us to estimate the density ratio without multiplicative ambiguity as long as $\nu>0$. By \emph{anchor}, we mean that class 0 acts as a fixed reference distribution, eliminating arbitrary scaling and ensuring \emph{identifiability}, which will become precise in Theorem~\ref{thm:fisher} below. We can take $\nu=0$ if $K\ge 2$ (recovering InfoNCE), but require $\nu>0$ in the $K=1$ case to avoid degeneracy.
More succinctly, we can write, for $z\neq 0$, 
\[
p(x_{1:K}|z) = \frac{\pdata(x_z)}{\pnoise(x_z)} \pnoise(x_1) \pnoise(x_2)\cdots \pnoise(x_K)
=\frac{\pdata(x_z)}{\pnoise(x_z)} p(x_{1:K}|z=0).
\]
In words, for $z\neq 0$, the class density is designed such that $x_z$ is drawn from $\pdata$, and the rest are from $\pnoise$.
We can write the marginal distribution over $x_{1:K}$ as
\begin{align*}
p(x_{1:K}) = \pnoise(x_1) \pnoise(x_2)\cdots \pnoise(x_K)\biggl(\frac{1}{K+\nu}\sum_{i=1}^K \frac{\pdata(x_i)}{\pnoise(x_i)} + \frac{\nu}{K+\nu} \biggr).
\end{align*}
By Bayes' rule, the posterior probability $p(z|x_{1:K})$ is 
\begin{align}
p(z|x_{1:K})=
\frac{p(x_{1:K}|z)p(z)}{p(x_{1:K})}
=\begin{cases}
\displaystyle\frac{\nu}{\nu + \sum_{i=1}^K \frac{\pdata(x_i)}{\pnoise(x_i)}} & \text{if }z=0\\
\displaystyle\frac{\frac{\pdata(x_z)}{\pnoise(x_z)}}{\nu + \sum_{i=1}^K \frac{\pdata(x_i)}{\pnoise(x_i)}} & \text{if }z\in[K]
\end{cases}.
\label{eq:def_true_cp}
\end{align}
This motivates us to parameterize our probabilistic classifier $p_\th(z|x_{1:K})$ in the form of
\begin{align}
p_\th(z|x_{1:K})
=\begin{cases}
\displaystyle\frac{\nu}{\nu + \sum_{i=1}^K r_\th(x_i)} & \text{if }z=0\\
\displaystyle\frac{r_\th(x_z)}{\nu + \sum_{i=1}^K r_\th(x_i)} & \text{if }z\in[K]
\end{cases}.
\label{eq:def_model_cp}
\end{align}
Applying the maximum likelihood estimation (MLE) principle, we can derive the population objective
\begin{align}
\Lc_{K;\nu}(\th)&\defeq -\frac{K}{K+\nu} \E_{\pdata(x_1)\pnoise(x_2)\cdots \pnoise(x_K)}\Biggl[\log \frac{r_\th(x_1)}{\nu + \sum_{i=1}^K r_\th(x_i)} \Biggr]\nonumber\\
&\qquad-\frac{\nu}{K+\nu} \E_{\pnoise(x_1)\pnoise(x_2)\cdots \pnoise(x_K)}\Biggl[\log \frac{\nu}{\nu + \sum_{i=1}^K r_\th(x_i)} \Biggr],
\label{eq:kway_obj}
\end{align}
since $\max_\th 
\E_{p(z)p(x_{1:K}|z)}[\log {p_\th(z|x_{1:K})}]
= \min_\th \Lc_{K;\nu}(\th).$
We call it the \emph{InfoNCE-anchor} objective.
Suggested by the name, when $K\ge 2$ and $\nu=0$, it boils down to InfoNCE. 
In another extreme, when $K=1$ and $\nu=1$, it becomes equivalent to the standard variational lower bound of Jensen--Shannon divergence (see Table~\ref{tab:mi_estimators}).
In the language of \emph{noise contrastive estimation},
this provides a unification of the standard NCE~\citep{Gutmann--Hyvarinen2012} ($K=1,\nu>0$), and the so-called \emph{ranking NCE} objectives~\citep{Ma--Collins2018} ($K=2,\nu=0$).

\textbf{Fisher Consistency.}
When $\nu>0$, it readily follows from the MLE principle that InfoNCE-anchor characterizes the density ratio $\frac{\pdata(x)}{\pnoise(x)}$ as its global minimizer in the population and nonparametric limit.
\begin{theorem}[Fisher consistency]
\label{thm:fisher}
Let $\th^*\defeq\arg\min_\th \Lc_{K;\nu}(\th)$ denote a global optimizer of the InfoNCE-anchor objective. 
Suppose that there exists $\th_0$ such that $r_{\th_0}(x)=\frac{\pdata(x)}{\pnoise(x)}$.
If $K\ge 1$ and $\nu>0$, $r_{\th^*}(x)=\frac{\pdata(x)}{\pnoise(x)}$ for almost every $x$ under $\pnoise$.
If $K\ge 2$ with $\nu=0$, there exists some constant $C>0$ such that $r_{\th^*}(x)=C \frac{\pdata(x)}{\pnoise(x)}$
for $\pnoise$-almost every $x$.
\end{theorem}

\subsection{Application: Divergence Estimation and Representation Learning} 
\label{sec:application}
We can apply the InfoNCE-anchor objective to estimate MI or to learn representation when given a joint distribution $p(x,y)$, in a similar way to InfoNCE~\citep{van-den-Oord--Li--Vinyals2018}.
That is, for each $y$, we can apply the InfoNCE-anchor for $\pdata(x)\gets p(x|y)$ and $\pnoise(x)\gets p(x)$.\footnote{An alternative approach is to set $(\pdata(x),\pnoise(x))\gets (p(x,y),p(x)p(y))$; see Appendix~\ref{app:sec:alternative_mi_estimation}.} For the final objective, we take an expectation over $y\sim p(y)$:
\begin{align*}
\Lc_{K;\nu}^{(1)}(\th)&\defeq \E_{p(y)}\Biggl[-\frac{K}{K+\nu} \E_{p(x_1|y)p(x_2)\cdots p(x_K)}\Biggl[\log \frac{r_\th(x_1,y)}{\nu + \sum_{i=1}^K r_\th(x_i,y)} \Biggr]\\
&\qquad\qquad\quad -\frac{\nu}{K+\nu} \E_{p(x_1)p(x_2)\cdots p(x_K)}\Biggl[\log \frac{\nu}{\nu + \sum_{i=1}^K r_\th(x_i,y)} \Biggr]\Biggr].
\end{align*}
When $\nu=0$ with $K\ge 2$, it boils down to the original InfoNCE, and minimizing $\Lc_{K;0}^{(1)}(\th)$ can only guarantee that for some function $C(y)$, $r_{\th^*}(x,y)=C(y)\frac{p(x,y)}{p(x)p(y)}$.
When applied to representation learning, the vanilla InfoNCE (i.e., with $\nu=0$) thus may lead to an undesirable behavior due to uncontrollable $C(y)$, whereas the anchor (i.e., $\nu>0$) can remove such degeneracy. 
However, in our representation learning experiment, we observe that the anchor does not lead to the improvement of downstream task performance; see Section~\ref{sec:exp_representation}.

With a minibatch of size $B$, we can implement the loss with anchor for $K=B-1$ as follows:
\begin{align*}
-\frac{K}{K+\nu}\frac{1}{B}\sum_{b=1}^B \log \frac{r_{bb}}{\nu+\sum_{j\in [B]\backslash \{b-1\}} r_{bj}}
-\frac{\nu}{K+\nu}\frac{1}{B}\sum_{b=1}^B \log\frac{\nu}{\nu+\sum_{j\in [B]\backslash \{b\}} r_{bj}}.
\end{align*}
We provide a pseudocode in Appendix~\ref{app:sec:code}.
The density ratio estimator is typically parameterized as $r_\th(x,y)\gets e^{c_\th(x,y)}$, where $c_\th(x,y)$ (the \emph{critic}) is often a neural network. In representation learning, common choices are the exponential form $r_\th(x,y)\gets e^{\frac{1}{\tau}\frac{\fv_\th(x)^\intercal \gv_\th(y)}{|\fv_\th(x)|2 |\gv\th(y)|2}}$ (see, \eg \citep{van-den-Oord--Li--Vinyals2018}) or the direct form $r_\th(x,y)\gets \frac{1}{\tau}\frac{\fv_\th(x)^\intercal \gv_\th(y)}{|\fv_\th(x)|2 |\gv\th(y)|2}$ (see, \eg \citep{HaoChen--Wei--Gaidon--Ma2021}), such that $\fv_\th(x)$ and $\gv_\th(y)$ are learned embeddings that approximate PMI or PD, respectively. Here, $\tau>0$ is a \emph{temperature} parameter.

\subsection{InfoNCE-anchor Interpolates DV and NWJ Bounds When \texorpdfstring{$K\to\infty$}{K Grows to Infinity}}
One may ask about the behavior of InfoNCE-anchor when $K\to\infty$. 
While we defer a rigorous statement (Theorem~\ref{thm:generalized_dv}) to Appendix~\ref{app:sec:asymptotic}, we remark that InfoNCE-anchor, by setting $\nu$ to vary as $K\to\infty$ such that $\nu/K\to \b$ for some $\b\ge 0$, we can show that InfoNCE-anchor behaves similar to a generalization of the DV bound, which can be rearranged to yield
\begin{align}
\E_{\pdata(x)}[\log r_\th(x)] - (\b+1)\log\biggl(\frac{\b}{\b+1} +\frac{1}{\b+1}\E_{\pnoise(x)}[r_\th(x)]\biggr)
\le D(\pdata~\|~\pnoise).
\label{eq:generalized_dv}
\end{align}
When $\b=0$, this boils down to the standard DV bound.
When $\b\to\infty$, the left-hand side becomes $\E_{\pdata(x)}[\log r_\th(x)] - \E_{\pnoise(x)}[r_\th(x)]+1
\le D(\pdata~\|~\pnoise)$,
which is the NWJ bound.
Moreover, we can even show that this bound \emph{monotonically} interpolates between the DV bound (tightest, $\b=0$) and the NWJ bound (loosest, $\b=\infty$).
A similar asymptotic behavior of InfoNCE (i.e., for $\nu=0$) was noted by \citet{Wang--Isola2020}, but specifically in the context of contrastive representation learning.

\subsection{Discussion on Existing Variants of InfoNCE Estimator}
In this section, we discuss two existing variants of InfoNCE, which were proposed in the effort of fixing the aforementioned issues of InfoNCE as the MI estimator. We highlight why they are insufficient as a fundamental fix, and how different from our proposal.

\textbf{$\a$-InfoNCE.}
\citet{Poole--Ozair--van-den-oord-Alemi--Tucker2019} proposed an alternative estimator called $\a$-InfoNCE, defined as
\begin{align*}
\Dc_{\a\text{-InfoNCE}}(\th)
&\defeq \E_{\pdata(x_1)\pnoise(x_2)\cdots\pnoise(x_K)}\Biggl[\log\frac{r_\th(x_1)}{\a r_\th(x_1) + \frac{1-\a}{K-1}\sum_{z=2}^K r_\th(x_z)}
\Biggr]
\end{align*}
for some $\a\in(0,\frac{1}{K}]$.
Note that setting $\a\gets\frac{1}{K}$ recovers the original InfoNCE bound.
For $\a<\frac{1}{K}$, this quantity can neither be understood as a loss for classification nor be a lower bound for $D(\pdata~\|~\pnoise)$.
\citet[Theorem~4.2]{Lee--Shin2022} claimed that $\a$-InfoNCE is a \emph{tight} variational lower bound for a $\a$-skew KL divergence $D(\pdata~\|~\a\pdata+(1-\a)\pnoise)$, that is, $\Dc_{\a\text{-InfoNCE}}(\th)
\le D(\pdata~\|~\a\pdata+(1-\a)\pnoise)$
and the equality can be achieved.
We find, however, the proof has a flaw and it can be only guaranteed that $\Dc_{\a\text{-InfoNCE}}(\th)
\ge \Dc_{\text{DV}}(\th;\pdata,\a\pdata+(1-\a)\pnoise)$, while the equality condition remains unclear.

\textbf{MLInfoNCE.}
\citet{Song--Ermon2020b} introduced
the \emph{multi-label InfoNCE (MLInfoNCE)} estimator defined as
\begin{align*}
\Dc_{\text{MLInfoNCE}}(\th)
&\defeq \E_{\prod_{w=1}^m\pdata(x_{w1})\prod_{z=2}^k\pnoise(x_{wz})}\Biggl[\sum_{w=1}^m \log\frac{r_\th(x_{w1})}{\sum_{w'=1}^m (r_\th(x_{w'1}) + \sum_{z=2}^k r_\th(x_{w'z}))}
\Biggr].
\end{align*}
\citet[Theorem~2]{Song--Ermon2020b} shows that $\Dc_{\text{MLInfoNCE}}(\th)\le D(\pdata~\|~\pnoise)$.
However, we note that this objective cannot be understood as a loss derived from a proper classification setup unlike InfoNCE-anchor.

\subsection{Extension with Proper Scoring Rules}
\label{sec:extensions}
In the classification setup of Eq.~\eqref{eq:def_class}, density ratio estimation reduces to estimating the class probability $p(z|x_{1:K})$ in Eq.~\eqref{eq:def_true_cp} via the model $p_\th(z|x_{1:K})$ in Eq.~\eqref{eq:def_model_cp}. The cross-entropy loss in Eq.~\eqref{eq:kway_obj} is a \emph{proper} scoring rule, ensuring that the optimized model recovers the true posterior. More generally, once density ratio estimation is cast as class probability estimation, \emph{any proper scoring rule} can be applied, yielding a broad family of consistent objectives.

Here we start with a general description of the proper scoring rules~\citep{Gneiting--Raftery2007,Dawid--Lauritzen--Parry2012}.
Let $\Zc$ be a discrete alphabet and let $\Ac$ be any alphabet.
Suppose that we have sample access to the underlying distribution $p(a,z)$ over $\Ac\times\Zc$.
The goal of class probability estimation (CPE)~\citep{Garcia--Williamson2012} is to estimate the underlying class probability $\etabdata\suchthat \Ac\to \Delta(\Zc)$, where $\etab(a)\defeq(p(z|a))_{z\in\Zc}$, using samples from $p(a,z)$.

To characterize a class probability estimator as the optimizer of an optimization problem, we consider a tuple of loss functions $\lambdab=(\lambda_z\suchthat \Delta(\Zc)\to \Real)_{z\in \Zc}$, which we call a \emph{scoring rule}
, whereby an \emph{action} $\hat{\etab}\suchthat \Ac\to\Delta(\Zc)$ incurs loss $\lambda_z(\hat{\etab}(a))$ for a data point $(a,z)$.
Then, we measure the performance of an action $\hat{\etab}$ by the expected loss $\E_{p(a,z)}[\lambda_z(\hat{\etab}(a))]$.
\begin{definition}[Proper scoring rules]
A scoring rule $\lambdab\suchthat \Delta(\Zc)\to \Real^\Zc$ is a vector-valued loss function. 
A scoring rule is said to be \emph{proper} if $\etabdata$ is optimal with respect to $\lambdab$, \ie for any distribution $p(a,z)$,
\[
\etabdata(\cdot)\in \arg\min_{\hat{\eta}\suchthat \Ac\to \Delta(\Zc)} \E_{p(a,z)}[\lambda_z(\hat{\etab}(a))].
\]
If $\etabdata$ is the \emph{unique} optimal solution with respect to $\lambdab$, then $\lambdab$ is said to be \emph{strictly} proper.
\end{definition}

We note that most (strictly) proper scoring rules can be induced by a (strictly) differentiable convex function. For the sake of exposition, let $\Zc=\{0,\ldots,M\}$ concretely.
Then, for a differentiable function $\Psi\suchthat\{1\}\times\Real_+^M\to\Real$, we define the $\Psi$-induced scoring rule as
\begin{align}\label{eq:def_induced_scoring_rule}
\lambdab^{\Psi}(\etab)
\defeq \begin{bmatrix}
\langle  \rhob, \nabla_{\rhob}\Psi(\rhob)\rangle - \Psi(\rhob)\\
(-\nabla_{\rhob}\Psi(\rhob))_{1:M}
\end{bmatrix}\Biggr|_{\rhob=(1,\frac{\eta_1}{\eta_0},\ldots,\frac{\eta_M}{\eta_0})}.
\end{align}
\begin{proposition}\label{prop:induced_loss}
If $\Psi$ is (strictly) convex and twice differentiable, $\lambdab^{\Psi}$ is (strictly) proper.
\end{proposition}
The canonical example is the log score, which results in InfoNCE-anchor in Eq.~\eqref{eq:kway_obj}.
We present some examples of proper scoring rules and the generating convex functions in Appendix~\ref{app:sec:examples_scores}.

Now, considering the classification setup in Eq.~\eqref{eq:def_class}, let $\lambdab=\lambdab^{\Psi}$ be a strictly proper scoring rule over discrete alphabet $\Zc=\{0,\ldots,K\}$, induced by a strictly convex function $\Psi\suchthat\Real_+^K\to\Real$.
Applying the scoring rule to evaluate the score of the class probability $p_\th(z|x_{1:K})$ (in Eq.~\eqref{eq:def_model_cp}) with respect to the underlying distribution $p(z)p(x_{1:K}|z)$, we can write the population objective (to be minimized) as
\begin{align*}
\Lc_{K;\nu}^{\Psi}(\etab_\th) &\defeq 
\E_{p(x_{1:K},z)}[\lambda_z(\etab_\th(x_{1:K}))],
\end{align*}
where we use $\etab_\th(x_{1:K})=(p_\th(z|x_{1:K}))_{z\in\Zc}$ to denote the class probability vector.
Let $\etabref(x_{1:K})$ denote the underlying class probability $(p(z|x_{1:K}))_{z\in\Zc}$.
The following statement subsumes Theorem~\ref{thm:fisher}. 
\begin{theorem}
\label{thm:bregman}
For $\nu> 0$,
\begin{align*}
\Lc_{K;\nu}^{\Psi}(\etab_\th)-\Lc_{K;\nu}^{\Psi}(\etabref)
&= \frac{\nu}{K+\nu} \E_{\pnoise(x_1)\pnoise(x_2)\cdots \pnoise(x_K)}\biggl[B_{\Psi}\biggl(\frac{\rb^*(x_{1:K})}{\nu}, \frac{\rb_\th(x_{1:K})}{\nu}\biggr)\biggr],
\end{align*}
where $\rv^*(x_{1:K}) \defeq 
\bigl(\frac{\pdata(x_{z})}{\pnoise(x_{z})}\bigr)_{z\in[K]}$ and $\rv_\th(x_{1:K}) 
\defeq 
\bigl(r_\th(x_z)\bigr)_{z\in[K]}$.
If $\Psi$ is convex, we have
\begin{align*}
-\Lc_{K;\nu}^{\Psi}(\etab_\th)
\le -\Lc_{K;\nu}^{\Psi}(\etabref)
&= \frac{\nu}{K+\nu} \E_{\pnoise(x_1)\pnoise(x_2)\cdots \pnoise(x_K)}\biggl[\Psi\biggl(\frac{\rb^*(x_{1:K})}{\nu}\biggr)\biggr].
\end{align*}
If $\Psi$ is (strictly) convex, the equality is achieved if (and only if) $r_\th(x)=\frac{\pdata(x)}{\pnoise(x)}$.
\end{theorem}

Beyond the consistency, this corollary shows that the DRE objective (with negation) can be understood as a variational lower bound of some divergence between $\pdata(x)$ and $\pnoise(x)$ induced by $\Psi$, defined as $\E_{\pnoise(x_1)\pnoise(x_2)\cdots \pnoise(x_K)}[\Psi(\frac{\rb^*(x_{1:K})}{\nu})]$.
This is analogous to that the InfoNCE-anchor objective in Eq.~\eqref{eq:kway_obj} is a variational lower bound of the $K$-way JSD $D_{K\textsf{-JS}}(\pdata,\pnoise)$.
We note that this extension can be viewed as a special application of the more general multi-distribution density ratio estimation studied by \citet{Yu--Jin--Ermon2021}, for the binary density ratio estimation. 

\textbf{Implementation.}
Similar to InfoNCE-anchor in Eq.~\eqref{eq:kway_obj},
this objective function can be simplified further if the scoring rule satisfies a mild symmetry condition; see Appendix~\ref{app:sec:bregman_implementation}.

\textbf{Alternative Characterization of Proper Scoring Rule.}
One minor limitation of the characterization in Theorem~\ref{thm:bregman} is that $\nu=0$ is not permitted as a special case, and thus InfoNCE cannot be subsumed.
In Appendix~\ref{app:sec:alternative_characterization_proper_scoring_rules}, we provide an alternative characterization of proper scoring rules, which can be related to the above formulation via the \emph{perspective transformation}, and admits $\nu=0$. 

\textbf{Special Cases.}
For the special case when $K=1$ and $\nu=1$, note that the right hand side becomes the $f$-divergence $D_f(\pdata~\|~\pnoise)$ when $f=\Psi$.
That is, the objective boils down to the standard variational lower bound on the $f$-divergence~\citep{Ngyuen--Wainwright--Jordan2010}, hence recovering the $f$-DIME objectives of \citet{Letizia--Novello--Tonello2024} and $f$-MICL objectives of \citet{Lu--Zhang--Sun--Guo--Yu2024}. 
We also note that the GAN-DIME and HD-DIME estimators in \citep{Letizia--Novello--Tonello2024} are essentially identical to the estimators proposed in \citep{Tsai--Zhao--Yamada--Morency--Salakhutdinov2020}.
More examples can be found in Appendix~\ref{app:sec:examples_dre_objs}.

\section{Experiments}
\label{sec:exp}
In this section, we show that InfoNCE-anchor outperforms existing estimators in MI estimation (Section~\ref{sec:exp_mi_estimation}) and downstream classification task (Section~\ref{sec:exp_kinase_ligand}). We also report a negative result: anchor variants do not improve the representation quality of InfoNCE in self-supervised representation learning tasks (Section~\ref{sec:exp_representation}). In all experiments we set $K=B-1$ and $\nu=1$ by default. 

\subsection{MI Estimation}
\label{sec:exp_mi_estimation}
We evaluate various neural MI estimators on structured and unstructured data using the benchmark suite of \citet{Lee--Rhee2024}.\footnote{GitHub: \url{https://github.com/kyungeun-lee/mibenchmark}}
Experiments cover three domains: multivariate Gaussian data, MNIST images, and BERT embeddings of IMDB reviews.
To control ground truth MI, the benchmark employs same-class sampling for positive pairs and a binary symmetric channel to inject controlled noise.
This allows systematic variation of MI from 2 to 10 bits in 2-bit increments.
Implementation details such as critic architectures and optimization setups are deferred to Appendix~\ref{app:sec:exp_details}.

\begin{figure}[ht]
    \centering

    \begin{subfigure}{\textwidth}
        \centering
        \includegraphics[width=\linewidth]{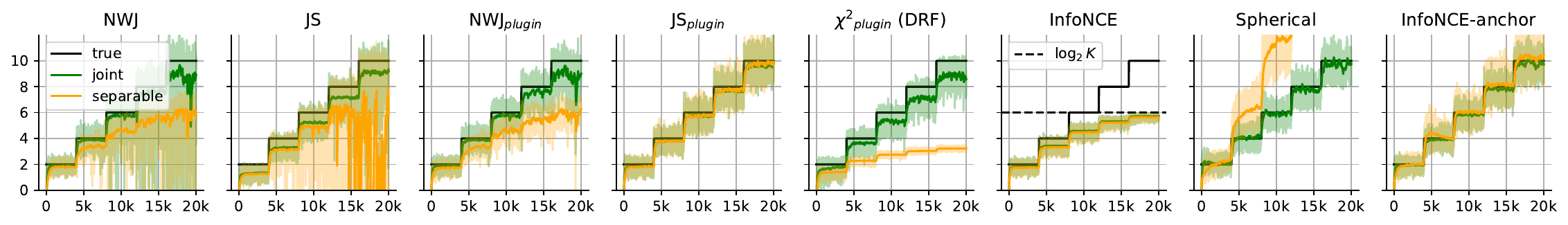}\vspace{-.25em}
        \caption{Gaussian with cubic transformation.}
    \end{subfigure}
    
    \begin{subfigure}{\textwidth}
        \centering
        \includegraphics[width=\linewidth]{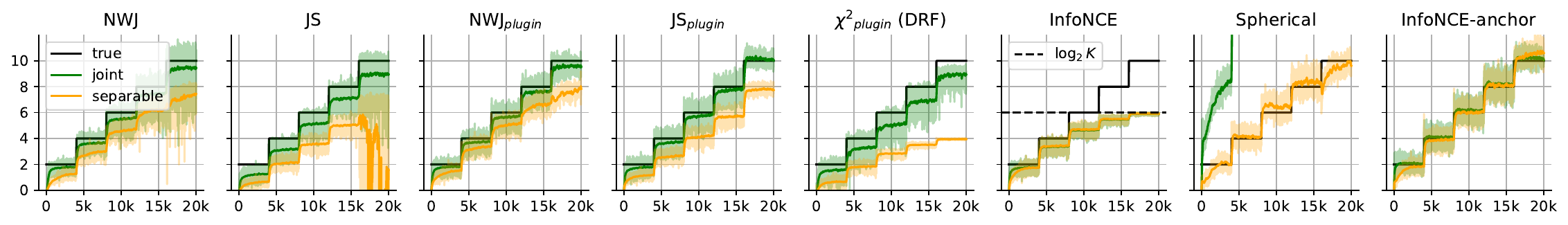}\vspace{-.25em}
        \caption{MNIST.}
    \end{subfigure}

    \begin{subfigure}{\textwidth}
        \centering
        \includegraphics[width=\linewidth]{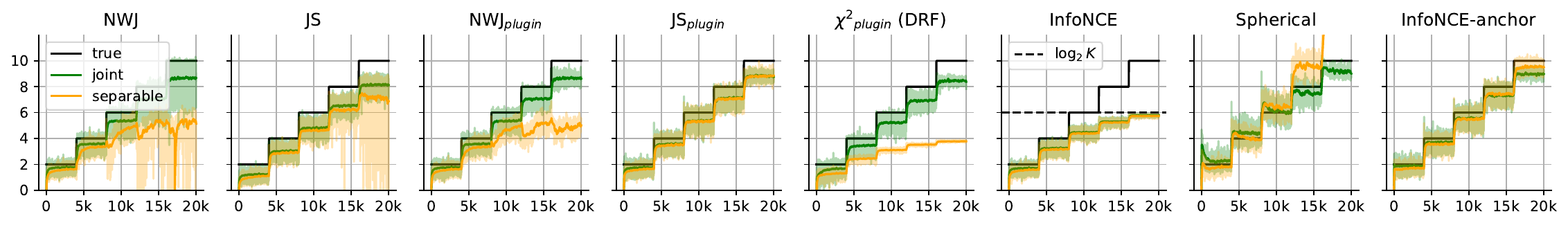}\vspace{-.25em}
        \caption{Texts.}
    \end{subfigure}

    \caption{Summary of MI estimation results on the standard benchmark. All experiments were done with batch size 64 and averaged over 20 random runs. Across all the cases, the proposed InfoNCE-anchor estimator (the rightmost column) consistently demonstrates low-bias, low-variance performance compared to the existing estimators. See Section~\ref{sec:exp_mi_estimation} for the experiment setup.}
    \label{fig:mi_estimation}
\end{figure}

Figure~\ref{fig:mi_estimation} summarizes the results.
InfoNCE-anchor tracks ground truth MI most closely across domains.
$\mathsf{JS}_{\mathsf{plugin}}$ (equivalent to InfoNCE-anchor with $K=1$, $\nu=1$) performs comparably on Gaussians but deteriorates on higher-dimensional tasks such as MNIST and texts, highlighting the value of large $K$.
We also evaluate $\mathsf{Spherical}$, an InfoNCE-anchor variant induced by the spherical scoring rule~\citep{Gneiting--Raftery2007}; see Table~\ref{tab:objs_sym_anchor} in Appendix~\ref{app:sec:examples_dre_objs}. Its inferior performance indicates that, despite the equivalence of strictly proper scoring rules, the log score remains the most effective in practice.
Additional results for Gaussian with varying batch sizes can be found in Appendix~\ref{app:sec:exp_details}.

\subsection{Protein Interaction Prediction}
\label{sec:exp_kinase_ligand}
As a further demonstration of the effectiveness of InfoNCE-anchor, we perform an experiment from a recent study by~\citet{Gowri--Lun--Klein--Yin2024}. In the work, the authors examined protein embeddings derived from a pretrained protein language model (pLM), the \emph{ProtTrans5} model~\citep{Elnaggar--Heinzinger--Dallago--Rehawi--Wang--Jones--Gibbs--Feher--Angerer--Steinegger2021}, and evaluated whether one can predict interactions between protein pairs~$(x, y)$, specifically, $(K,T)=$ (kinase, target) and $(L,R)=$ (ligand, receptor) pairs in the considered setting.
The interaction labels are from the \emph{OmniPath} database~\citep{Turei--Valdeolivas--Gul--Palacio-Escat--Klein--Ivanova--Olbei--Glabor--Theis--Modos2021}.
We ran the experiment following the same setup, with estimating the PMI using the JS, InfoNCE-anchor, and a few other variational approaches, and using them to decide whether interaction exists by thresholding the PMI of a given pair.

\begin{figure}[ht]
\label{fig:kinase_ligand_summary}
    \begin{subfigure}[c]{0.58\textwidth}
        \centering
        \includegraphics[width=\linewidth]{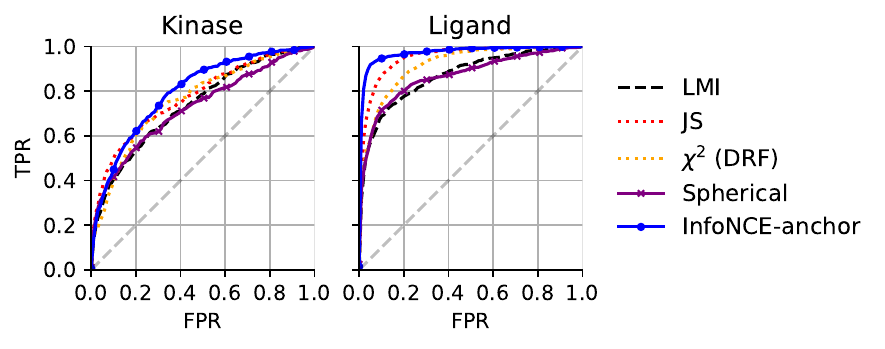}
        \caption{ROC curves for each problem instance.}
        \label{fig:roc}
    \end{subfigure}
    \begin{subfigure}[c]{0.40\textwidth}
        \centering
        \caption{Summary of AUROC performance on the Kinase and Ligand benchmarks
        (mean$_{\pm\text{std}}$ computed over 20 runs).}
        \label{tab:auc-kinase-ligand}
        \vspace{0.5em} %
        \scriptsize
        \begin{tabular}{lll}
            \toprule
            Objective            & Kinase       & Ligand      \\
            \midrule
            LMI                  & $0.74_{\,\pm0.08}$ & $0.87_{\,\pm0.04}$ \\
            $\chi^{2}$ (DRF)     & $0.76_{\,\pm0.07}$ & $0.92_{\,\pm0.03}$ \\
            JS                   & $0.77_{\,\pm0.08}$ & $0.95_{\,\pm0.02}$ \\
            InfoNCE-anchor       & $\mathbf{0.80_{\,\pm0.06}}$ & $\mathbf{0.97_{\,\pm0.01}}$ \\
            Spherical            & $0.73_{\,\pm0.07}$ & $0.87_{\,\pm0.05}$ \\
            \bottomrule
        \end{tabular}
    \end{subfigure}
    \caption{Summary of the protein interaction prediction experiment.}
\end{figure}

Figure~\ref{fig:roc} and Figure~\ref{tab:auc-kinase-ligand} summarize the results.
As shown in Figure~\ref{fig:roc}, InfoNCE-anchor shows the best prediction results for both problem instances, while the JS estimator, which is a special case of InfoNCE-anchor when $K=1$ and $\nu=1$, performs second best. This again demonstrates the practical benefit of large $K$ for accurate density ratio estimation. We also recall that the standard InfoNCE objective cannot be even applied to this scenario, as it only estimates PMI up to a multiplication with an arbitrary function $C(y)$ discussed in Section~\ref{sec:application}.
We include the histograms of learned PMI values (Figure~\ref{fig:histograms}) as well as the ROC curves of each estimator for different runs (Figure~\ref{fig:roc_each_problem}) in Appendix~\ref{app:sec:exp_details}.

\subsection{Self-Supervised Representation Learning}
\label{sec:exp_representation}

In earlier sections we showed that InfoNCE-anchor improves MI estimation and downstream tasks using the learned density ratio model. A natural question is whether this benefit carries over to self-supervised learning (SSL), where InfoNCE is the standard objective.
We therefore pretrain a ResNet-18 on CIFAR-100 using the \texttt{solo-learn} framework~\citep{JMLR:v23:21-1155}, comparing several contrastive objectives under identical settings (batch size $B=256$, same optimizer), and evaluate representations via linear probing.

\begin{table}[ht]
\centering
\caption{Linear probing accuracy (\%) after SSL pretraining. We used PD parameterization for Spherical and $\chi^2$. Detailed setups can be found in Appendix~\ref{app:sec:exp_details}.}
\label{tab:ssl}
\begin{tabular}{lcccccc}
\toprule
Objective 
& InfoNCE & InfoNCE-anchor & Spherical & JS & $\chi^2$ \\
\midrule
Top-1 accuracy & 65.98 & 65.74 & 4.33 & 61.69 & 65.59 \\
Top-5 accuracy & 89.69 & 89.24 & 17.91 & 87.33 & 88.4 \\
\bottomrule
\end{tabular}
\end{table}

Table~\ref{tab:ssl} shows that InfoNCE continues to yield the strongest representations. Adding the anchor, despite improving density ratio estimation, does not translate into better SSL performance. This suggests that the uncontrollable multiplicative factor $C(y)$ in InfoNCE is either nearly constant or irrelevant for representation learning.
JS performs poorly, highlighting the importance of large $K$, while spherical scores collapse entirely, likely due to unfavorable optimization dynamics.
Overall, these findings indicate that neither accurate MI estimation nor exact density ratio recovery is essential for high-quality representations. What matters in SSL appears to be the factorization of PMI, the benefit of large $K$, and the favorable optimization properties with the log score.

\section{Concluding Remarks}

We revisited InfoNCE and showed it is not a consistent MI estimator but a variational bound of some other divergence. We introduce InfoNCE-anchor, a simple fix enabling consistent density ratio estimation within a unified scoring-rule framework. 
InfoNCE-anchor sets new state-of-the-art MI estimation benchmarks and aids predictive tasks, though it does not improve SSL performance, highlighting that accurately estimating MI is not essential for representation quality~\citep{Tschannen--Djolonga--Rubenstein--Gelly--Lucic2020}.
We hope our work clarifies the role of InfoNCE and MI estimation in contrastive learning.

\newpage
\bibliographystyle{iclr2026_conference}
\bibliography{ref}
\newpage

\appendix
\addtocontents{toc}{\protect\StartAppendixEntries}
\listofatoc

\appendix

\section{Alternative Approach to MI Estimation}
\label{app:sec:alternative_mi_estimation}

As alluded to earlier in Section~\ref{sec:application}, we can construct an alternative consistent objective function for estimating the pointwise dependence $\frac{p(x,y)}{p(x)p(y)}$.
That is, applying the InfoNCE-anchor framework for $\pdata(x)\gets p(x,y)$ and $\pnoise(x,y)\gets p(x)p(y)$, we obtain
\begin{align*}
\Lc_{K;\nu}^{(2)}(\th)&\defeq 
-\frac{K}{K+\nu} \E_{p(x_1,y_1)p(x_2)p(y_2)\cdots p(x_K)p(y_K)}\Biggl[\log \frac{r_\th(x_1,y_1)}{\nu + \sum_{i=1}^K r_\th(x_i,y_i)} \Biggr]\\
&\qquad-\frac{\nu}{K+\nu} \E_{p(x_1)p(y_1)p(x_2)p(y_2)\cdots p(x_K)p(y_K)}\Biggl[\log \frac{\nu}{\nu + \sum_{i=1}^K r_\th(x_i,y_i)} \Biggr].
\end{align*}
While this version results in a different, yet consistent objective function, it is not preferable over the discussed approach in practice.

With this approach, when $\nu=0$, minimizing $\Lc_{K;0}^{(2)}(\th)$ guarantees that for some $C>0$, 
\begin{align*}
r_{\th^*}(x,y)=C\frac{p(x,y)}{p(x)p(y)}.
\end{align*}
This is a guarantee analogous to the MLInfoNCE~\citep{Song--Ermon2020b}.

\section{Pseudocode for InfoNCE-anchor}
\label{app:sec:code}
Here, we provide a pseudocode for the PyTorch implementation of the InfoNCE-anchor objective function.

\begin{lstlisting}[]
def infonce_with_anchor(scores, nu=1.0):
    """
    scores: [B, B] tensor where scores[i, j] = f(x_i, y_j)
    nu: prior smoothing hyperparameter
    """
    assert nu > 0.
    device = scores.device
    B = scores.size(0)
    K = B - 1
    
    # joint term
    mask = torch.zeros(B, B, device=device)
    i = torch.arange(1, B + 1)
    mask[i - 1, i - 2] = -torch.inf    
    scores_aug = torch.cat([
        np.log(nu) * torch.ones(B, 1, device=device), 
        mask + scores], dim=1)  # augmented score
    joint_term = - (scores.diag().mean() - scores_aug.logsumexp(dim=1).mean())
    
    # independent term
    neg_inf_diag_mask = torch.zeros(B, B, device=device).fill_diagonal_(-torch.inf)
    scores_aug_neg = torch.cat([
        np.log(nu) * torch.ones(B, 1, device=device), 
        neg_inf_diag_mask + scores
        ], dim=1)  # negative augmented score
    marginal_term = - (np.log(nu) - scores_aug_neg.logsumexp(dim=1).mean())

    return (K / (K + nu)) * joint_term + (nu / (K + nu)) * marginal_term
\end{lstlisting}

\section{Asymptotic Behavior of InfoNCE-Anchor}
\label{app:sec:asymptotic}
\begin{theorem}
\label{thm:generalized_dv}
If $\nu/K\to \b$ as $K\to\infty$ for some $\b\ge0$, then
\begin{align*}
\lim_{K\to\infty} \biggl(-\Lc_{K;\nu}(\th)+\frac{K\log K}{K+\nu}\biggr)
&= \frac{\b}{\b+1}\log \b +\frac{1}{\b+1}\E_{\pdata(x)}[\log r_\th(x)] - \log (\b+\E_{\pnoise(x)}[r_\th(x)])\\
&\le 
\frac{\b}{\b+1}\log\b
+\frac{1}{\b+1}D(\pdata~\|~\pnoise)
-\log(\b+1).
\end{align*}
The equality holds if and only if $r_\th(x)=\frac{\pdata(x)}{\pnoise(x)}$ when $\b>0$, and $r_\th(x)=C\frac{\pdata(x)}{\pnoise(x)}$ for some $C>0$ when $\b=0$.
\end{theorem}
Rewriting as a lower bound on the KL divergence, we have Eq.~\eqref{eq:generalized_dv}.

\section{Decision-Theoretic Treatment of Proper Scoring Rules}

This section serves as a decision-theoretic foundation on proper scoring rules for conditional probability estimation, which is essential to proving the main statements in Section~\ref{sec:extensions}, i.e., Proposition~\ref{prop:induced_loss} and Theorem~\ref{thm:bregman}.
Appendix~\ref{app:sec:one_to_one} is marked with asterisk, which is included for completeness and can be safely skipped in the first reading.

\subsection{Preliminaries and Definitions}
\label{app:sec:scoring_rules_prelim}

We first note that 
\[
\E_{p(a,z)}[\lambda_z(\hat{\etab}(a))]
=\E_{p(a)}\biggl[\sum_{z=0}^M \etabdata(a)\lambda_z(\hat{\etab}(a))\biggr]
=\E_{p(a)}[\langle\etabdata(a),\lambdab(\hat{\etab}(a))\rangle],
\]
which implies that we only need to study the \emph{conditional} problem for each $a$, without the expectation over $a\sim p(a)$.
We define the \emph{(conditional) risk} of $\hat{\eta}\in\Delta(\Zc)$ with respect to $\etabref\in\Delta(\Zc)$ as
\[
d^{\lambdab}(\hat{\etab}~\|~\etabref)\defeq
\sum_{z\in\Zc} (\etabref)_z\lambda_z(\hat{\etab})
=\langle \etabref, \lambdab(\hat{\etab})\rangle.
\]
In particular, we denote 
\[
f^{\lambdab}(\etabref)\defeq d^{\lambdab}(\etabref~\|~\etabref)
\]
and call the \emph{pointwise Bayes risk} with respect to $\etabref$, since we can write the Bayes-optimal risk as
\[
\E_{p(a,z)}[f^{\lambdab}(\etabdata(a))]
=\E_{p(a,z)}[d^{\lambdab}(\etabdata(a)~\|~\etabdata(a))] 
= \min_{\hat{\etab}\suchthat \Ac\to \Delta(\Zc)}\E_{p(a,z)}[\lambda_z(\hat{\etab}(a))],
\]
given that $\lambdab$ is proper.

Since the propriety of a scoring rule is independent of the distribution $p(a)$ over $\Ac$ as alluded to earlier, we can restate the definition of propriety as follows.
We define and denote the \emph{regret} of $\hat{\etab}$ with respect to $\etabref$ under $\lambdab$ as
\[
\Reg^{\lambdab}(\hat{\etab}~\|~\etabref)
\defeq d^{\lambdab}(\hat{\etab}~\|~\etabref)-d^{\lambdab}(\etabref~\|~\etabref).
\]

\begin{definition}
A loss-function tuple $\lambdab$ is said to be \emph{proper} if 
$\Reg^{\lambdab}(\hat{\etab}~\|~\etabref)\ge 0$ for any $\hat{\etab},\etabref$ and $\Reg^{\lambdab}(\etabref~\|~\etabref)= 0$ for any $\etabref$.
A loss-function tuple $\lambdab$ is said to be \emph{strictly proper} if it is proper and
$\Reg^{\lambdab}(\hat{\etab}~\|~\etabref)=0$ if and only if $\hat{\etab}=\etabref$, for any $\etabref$.
\end{definition}

We now state the characterization of differentiable (strictly) proper loss functions.
If $\lambdab$ is differentiable,
we let 
\begin{align}
\gb^{\lambdab}(\etabref)
\defeq \nabla_{\hat{\etab}} d^{\lambdab}(\hat{\etab}~\|~\etabref)|_{\hat{\etab}=\etabref}
=(\langle\nabla_j\lambdab(\etabref), \etabref\rangle)_{j\in\Zc} = \mathbf{J}\lambdab(\etabref)\etabref,
\label{eq:def_generating_function}
\end{align}
which is the gradient of the pointwise risk function $\hat{\etab}\mapsto d^{\lambdab}(\hat{\etab}~\|~\etabref)$ at $\hat{\etab}=\etabref$. 
Here, $\mathbf{J}\lambdab(\etabref)\in\Real^{m\times m}$ denotes the Jacobian of the matrix $\lambdab\suchthat \Delta(\Zc)\to \Real^m$, \ie
\[
(\mathbf{J}\lambdab(\etabref))_{ij}\defeq \frac{\partial\lambda_j(\etabref)}{\partial \etaref_i}.
\]

\begin{theorem}
\label{thm:proper_loss}
A scoring rule $\lambdab=(\lambda_z\suchthat \Delta(\Zc)\to\Real)_{z\in\Zc}$ is (strictly) proper if and only if (1) the pointwise Bayes risk function $\etab\mapsto f^{\lambdab}(\etab)$ is (strictly) concave over $\Delta(\Zc)$ and (2) $\gb^{\lambdab}(\etabref)=\zerob$ for any $\etabref\in\Delta(\Zc)$. 
\end{theorem}

To prove the theorem, we first state a key technical lemma. 
Given a differentiable function $f\suchthat V\to \Real$ over a subset $V$ of an Euclidean space, we define the \emph{Bregman distortion} $B_f\suchthat V\times V\to \Real$ as
\[
B_f(u,v)\defeq f(u)-f(v)-\langle \nabla f(v),u-v\rangle.
\]
If $f$ is convex, $B_f(x,z)$ is called the \emph{Bregman divergence} generated by $f$.
Finally, let $\fbar^{\lambdab}(\etab)\defeq -f^{\lambdab}(\etab)$ denote the negative pointwise Bayes risk.
\begin{lemma}
For any $\hat{\etab},\etabref$, we have
\label{lem:proper_loss}
\[
\Reg^{\lambdab}(\hat{\etab}~\|~\etabref)
=B_{\fbar^{\lambdab}}(\etabref,\hat{\etab}) + \langle \gb^{\lambdab}(\hat{\etab}),\hat{\etab}-\etabref\rangle.
\]
\end{lemma}
\begin{proof}
By chain rule, we have
\[
\nabla f^{\lambdab}(\hat{\etab})
=\lambdab(\hat{\etab})+\gb^{\lambdab}(\hat{\etab}).
\]
Therefore, by the definition of Bregman distortion, we have
\begin{align*}
B_{\fbar^{\lambdab}}(\etabref,\hat{\etab})
&=-B_{f^{\lambdab}}(\etabref,\hat{\etab})\\
&=-f^{\lambdab}(\etabref)+f^{\lambdab}(\hat{\etab})+\langle\nabla f^{\lambdab}(\hat{\etab}),\etabref-\hat{\etab}\rangle\\
&=-d^{\lambdab}(\etabref~\|~\etabref) + d^{\lambdab}(\hat{\etab}~\|~\hat{\etab})
-\langle \lambdab(\hat{\etab})+\gb^{\lambdab}(\hat{\etab}),\etabref-\hat{\etab}\rangle\\
&= -d^{\lambdab}(\etabref~\|~\etabref) + d^{\lambdab}(\hat{\etab}~\|~\hat{\etab})
+ d^{\lambdab}(\hat{\etab}~\|~\etabref)
- d^{\lambdab}(\hat{\etab}~\|~\hat{\etab})
+\langle\gb^{\lambdab}(\hat{\etab}),\etabref-\hat{\etab}\rangle\\
&= \Reg^{\lambdab}(\hat{\etab}~\|~\etabref)
+\langle\gb^{\lambdab}(\hat{\etab}),\etabref-\hat{\etab}\rangle,
\end{align*}
which concludes the proof.
\end{proof}
The first condition~(1) imposes that the estimation problem becomes (strictly) not easier as we \emph{mix} the class probabilities. The second condition~(2) formalizes that if $\etabref$ is the underlying class probability, then $\etab=\etabref$ is a local minimizer of the conditional risk function $\etab\mapsto d^{\lambdab}(\etab~\|~\etabref)$.

Now we are ready to prove Theorem~\ref{thm:proper_loss}.
\begin{proof}[Proof of Theorem~\ref{thm:proper_loss}]
We first prove the only-if direction. If $\lambdab$ is proper, then $d^{\lambdab}(\hat{\etab}~\|~\etabref)\ge d^{\lambdab}(\etabref~\|~\etabref)$ for any $\hat{\etab}$ and $\etabref$ by definition. 
That is, $\hat{\etab}\mapsto d^{\lambdab}(\hat{\etab}~\|~\etabref)$ is stationary at $\hat{\etab}=\etabref$, and thus the gradient $\gb^{\lambdab}(\etab)=\nabla_{\hat{\etab}} d^{\lambdab}(\hat{\etab}~\|~\etab)|_{\hat{\etab}=\etab}=0$ for any $\etab$.
Now, by the identity in Lemma~\ref{lem:proper_loss}, we have $B_{\fbar^{\lambdab}}(\etabref,\hat{\etab})=\Reg^{\lambdab}(\hat{\etab}~\|~\etabref)\ge 0$ for any $\etabref,\hat{\etab}$, which implies that the function $f^{\lambdab}=-\fbar^{\lambdab}$ is concave.
Further, if $\lambdab$ is strictly proper, then $B_{\fbar^{\lambdab}}(\etabref,\hat{\etab})=\Reg^{\lambdab}(\hat{\etab}~\|~\etabref)>0$ for any $\hat{\etab}\neq \etabref$, which implies that $f^{\lambdab}$ is strictly concave.

For the converse, \ie the if direction,
we can directly apply the identity in Lemma~\ref{lem:proper_loss} and
conclude 
$\Reg^{\lambdab}(\hat{\etab}~\|~\etabref)\ge B_{\fbar^{\lambdab}}(\etabref,\hat{\etab})\ge 0$ by the convexity of $\fbar^{\lambdab}$. It is clear that $\lambdab$ is strictly proper if $\fbar^{\lambdab}$ is strictly concave.
\end{proof}

\subsection{From Loss Function to Generating Function}
Given a loss function $\lambdab$, 
we define a corresponding \emph{generating function}
\[
\Psi^{\lambdab}(\rhob)\defeq -\langle \rhob, \lambdab(\etab)\rangle
\]
for $\rhob\in \{1\}\times \Real_+^M$, so that we can write the pointwise Bayes risk at $\etabref$ as
\[
f^{\lambdab}(\etabref)
= d^{\lambdab}(\etabref~\|~\etabref)
= -\etaref_0 \Psi^{\lambdab}(\rhobref).
\]
Then, it is easy to check that
\begin{proposition}
If $\lambdab$ is (strictly) proper, $\rhob\mapsto \Psi^{\lambdab}(\rhob)$ is (strictly) convex.
\end{proposition}
\begin{proof}
If $\lambdab$ is (strictly) proper, then the negative pointwise Bayes risk function $\etabref\mapsto -f^{\lambdab}(\etabref) 
=\etaref_0 \Psi^{\lambdab}(\rhobref)$ is (strictly) convex by Theorem~\ref{thm:proper_loss}.
Since the mapping is a perspective function of $\rhob\mapsto \Psi^{\lambdab}(\rhob)$, $\Psi^{\lambdab}$ must be (strictly) convex.
\end{proof}

\begin{remark}[From generating function to loss function]
Conversely, we can define a loss function from a differentiable function $\Psi\suchthat\{1\}\times \Real_+^M$ as follows:
\[
\lambdab^{\Psi}(\etab)
\defeq \begin{bmatrix}
\langle \nabla \Psi(\rhob), \rhob\rangle - \Psi(\rhob)\\
-\nabla \Psi(\rhob)_{1:M}
\end{bmatrix},
\]
so that we can write 
the pointwise Bayes risk at $\etabref$ as
\[
f^{\lambdab^{\Psi}}(\etabref)
= d^{\lambdab^{\Psi}}(\etabref~\|~\etabref)
= -\etaref_0 \Psi(\rhobref).
\]
\end{remark}

\subsection{One-to-One Correspondence\texorpdfstring{$^*$}{*}}
\label{app:sec:one_to_one}
A natural question to ask is whether $\lambdab\mapsto \Psi^{\lambdab}$ and $\Psi\mapsto \lambdab^{\Psi}$ are inverse mappings each other.
Indeed, we have the following propositions.
\begin{proposition}\label{prop:loss_ftn_loss}
\[
\lambdab^{\Psi^{\lambdab}}(\etab)
= \lambdab(\etab) - \langle \etab, \gb^{\lambdab}(\etab)\rangle\ones + \gb^{\lambdab}(\etab).
\]
\end{proposition}
Hence, in particular, if $\lambdab$ is proper, it readily follows that
$\lambdab^{\Psi^{\lambdab}}(\etab)\equiv \lambdab(\etab)$.
\begin{proof}[Proof of Proposition~\ref{prop:loss_ftn_loss}]
First, we consider $z\in\{1,\ldots,M\}$. 
Note that
\begin{align*}
\frac{\partial\lambda_z(\etab)}{\partial \rho_z} 
&= \Bigl\langle \frac{\partial}{\partial\rho_z} \frac{(1,\rho_1,\ldots,\rho_M)}{1+\rho_1+\ldots+\rho_M}, \nabla \lambda_z (\etab)\Bigr\rangle \\
&= \eta_0\langle -\etab+ \eb_Y, \nabla \lambda_z (\etab)\rangle,
\end{align*}
for any $z=0,1,\ldots,M$.
Thus, we have 
\begin{align*}
\sum_{z=0}^M \rho_z\frac{\partial\lambda_z(\etab)}{\partial \rho_z} 
&= \eta_0 \sum_z \rho_z 
\langle -\etab+ \eb_Y, \nabla \lambda_z (\etab)\rangle\\
&= \Bigl\langle -\etab+ \eb_Y, \sum_z\eta_z \nabla \lambda_z (\etab)\Bigr\rangle\\
&= \langle -\etab+ \eb_Y, \gb^{\lambdab}(\etab)\rangle.
\end{align*}
This implies that
\begin{align*}
\lambda_z^{\Psi^{\lambdab}}(\etab) 
&=\frac{\partial\Psi^{\lambdab}(\etab)}{\partial \rho_z}\\
&= -\frac{\partial\lambda_0(\etab)}{\partial \rho_z} -\lambda_z(\etab) - \sum_{z=1}^M \rho_z \frac{\partial\lambda_z(\etab)}{\partial \rho_z}\\
&= -\lambda_z(\etab) - \sum_{z=0}^M \rho_z \frac{\partial\lambda_z(\etab)}{\partial \rho_z}\\
&= -\lambda_z(\etab) - \langle \etab+\eb_z, \gb^{\lambdab}(\etab)\rangle.
\end{align*}

We now consider $z=0$. 
Observe that 
\begin{align*}
\langle \nabla\Psi^{\lambdab}(\rhob), \rhob\rangle 
&= \sum_{z=1}^M \rho_z \frac{\partial\Psi^{\lambdab}(\rhob)}{\partial\rho_z}\\
&= -\sum_{z=1}^M \rho_z\lambda_z(\etab) - \langle \etab+\eb_0, \gb^{\lambdab}(\etab)\rangle.
\end{align*}
Hence, we have 
\begin{align*}
\lambda_0^{\Psi^{\lambda}}(\etab)
&= \langle \nabla\Psi^{\lambdab}(\rhob), \rhob\rangle - \Psi^{\lambdab}(\rhob)\\
&= \lambda_0(\etab) + \sum_{z=1}^M \rho_z\lambda_z(\etab) - \sum_{z=1}^M \rho_z\lambda_z(\etab) - \langle \etab+\eb_0, \gb^{\lambdab}(\etab)\rangle\\
&= \lambda_0(\etab)- \langle \etab+\eb_0, \gb^{\lambdab}(\etab)\rangle.
\end{align*}
This concludes the proof.
\end{proof}

The following statement asserts that the generating function induced by the induced loss function of a generating function corresponds to the original generating function.
\begin{proposition}
\label{prop:psi_to_lamb_to_psi}
\[
\Psi^{\lambdab^\Psi}(\rhob) \equiv \Psi(\rhob).
\]
\end{proposition}
\begin{proof}
By definition, it is easy to check that
\begin{align*}
\Psi^{\lambda^{\Psi}}(\rhob)
&= -\langle\rhob,\lambdab^{\Psi}(\rhob)\rangle \\
&= -(\langle \nabla\Psi(\rhob),\rhob\rangle - \Psi(\rhob)) + \sum_{z=1}^M \rho_z \frac{\partial\Psi(\rhob)}{\partial \rho_z} \\
&= \Psi(\rhob).\qedhere
\end{align*}
\end{proof}

Therefore, there is a one-to-one correspondence between (strictly) proper loss functions $\{\lambdab\suchthat\Delta(\Zc)\to\Real^{\Zc}\}$ and (strictly) convex functions $\{\Psi\suchthat \{1\}\times\Real_+^{M}\to \Real\}$.

\subsection{Connection to Bregman Divergences}
\label{app:sec:bregman_connection}
Note the following proposition.
\begin{proposition}
\[
B_{f^{\lambdab}}(\etab^*,\etab)
= -\eta_0^* B_{\Psi^{\lambdab}}(\rhob^*, \rhob).
\]
\end{proposition}

The following corollary reveals that any CPE objective function induced by a proper scoring rule can be understood as a Bregman divergence minimization.
\begin{corollary}
If $\lambdab$ is proper, then
\[
\Reg^{\lambdab}(\etab ~\|~ \etab^*) = \eta_0^* B_{\Psi^{\lambdab}}(\rhob^*, \rhob).
\]
\end{corollary}

In other words, it shows that a proper loss function $\lambdab$ acts only through the form of the Bregman divergence $B_{\Psi^{\lambdab}}(\cdot,\cdot)$.
In other words, $\lambdab$ and $\lambdab'$ are equivalent CPE loss functions if $B_{\Psi^{\lambdab}}(\cdot,\cdot) \equiv B_{\Psi^{\lambdab'}}(\cdot,\cdot)$.
This defines an equivalence class in the set of loss functions 
\[
\Lambda(\Psi)\defeq\{\lambdab\suchthat \Delta(\Zc) \to\Real^{\Zc} | B_{\Psi^{\lambdab}}(\cdot,\cdot)\equiv B_{\Psi}(\cdot,\cdot)\}.
\]
We know that this set is always not empty, since Proposition~\ref{prop:psi_to_lamb_to_psi} implies that
\[
\lambdab^{\Psi}\in \Lambda(\Psi).
\]
Consider a subset
\[
\Lambda_o(\Psi)\defeq\{\lambdab\suchthat \Delta(\Zc) \to\Real^{\Zc} | \lambdab\in\Lambda(\Psi), \gb^{\lambdab}(\etab)\equiv \zerob\}.
\]
The loss functions in this subset can be thought as \emph{canonical} functions, as we require $\gb^{\lambdab}(\etab)\equiv \zerob$ to check propriety in Theorem~\ref{thm:proper_loss}.
Note that  
\[
\lambdab^{\Psi}\in \Lambda_o(\Psi),
\]
since Lemma~\ref{lem:gradient_induced_loss} establishes that $\gb^{\lambdab^{\Psi}}(\etab)\equiv \zerob$.
A small open question is whether $\lambdab^{\Psi}$ is an unique element of $\Lambda_o(\Psi)$.

\begin{remark}[Properization]
We remark that for any $\lambdab\in\Lambda(\Psi)$, we can map it to another element $\lambdab'\in\Lambda_o(\Psi)$, by defining it as
\[
\lambdab'(\etab) \defeq \lambdab(\etab) + \gb^{\lambdab}(\etab) - \langle \etab, \gb^{\lambdab}(\etab)\rangle \ones.
\]
It is easy to check that $\lambdab'\in \Lambda_o(\Psi)$ indeed.
One can think of this as a \emph{properization} of a loss function $\lambdab$, since for a convex function $\Psi$, any loss function $\lambdab\in\Lambda(\Psi)$ can be made into a proper loss $\lambdab'\in\Lambda_o(\Psi)$.
\end{remark}

\subsection{Examples of Proper Scoring Rules}
\label{app:sec:examples_scores}

We first start with proper \emph{binary} scoring rules; see Table~\ref{tab:binary_scoring_rules}. Most of the examples can be found from \citep{Gneiting--Raftery2007}. We refer to rules generated from the $\Psi$-induced scoring rules (Eq.~\eqref{eq:def_induced_scoring_rule}) by \emph{asymmetric} scoring rules, and the $\Phi$-induced rules (Eq.~\eqref{eq:def_induced_scoring_rule_phi}) by \emph{symmetric} rules.
\begin{table}[ht]
    \centering
    \caption{Examples of strictly proper binary scoring rules. 
    }%
    \small
    \begin{tabular}{l l l l l}
        \toprule
        Asymmetric scoring rule & $\Psi(\rho)$ 
        & \makecell[l]{ $\lambda_0^{\Psi}(\etab),\lambda_1^{\Psi}(\etab)$ (see Eq.~\ref{eq:def_induced_scoring_rule})} 
        \\       
        \midrule
        \makecell[l]{KLIEP~\citep{Sugiyama--Suzuki--Nakajima--Kashima--von-Bunau--Kawanabe2008kliep}
        } 
        & $\rho\log\rho$
        & $\frac{1}{\eta_0},
        -\log\frac{\eta_1}{\eta_0}$
        \\
        \makecell[l]{Robust DRE~($\a\notin\{0,1\}$)\\
        ~~~\citep{Sugiyama--Suzuki--Kanamori2012}} 
        & $\frac{\rho^\a}{\a(\a-1)}$ or $\frac{\rho^\a-\rho}{\a(\a-1)}$
        & $\frac{1}{\a}\frac{\eta_0^\a+\eta_1^\a}{\eta_0^\a} +\frac{1}{\a(\a-1)},
        \frac{1}{1-\a}(\frac{\eta_1}{\eta_0})^{\a-1}$
        \\
        Inverse log 
        & $-\log\rho$
        & $\log\frac{\eta_1}{\eta_0}-1,
        \frac{\eta_0}{\eta_1}$
        \\
        \toprule
        Symmetric scoring rule & $\Phi(\eta_0,\eta_1)$ 
        & \makecell[l]{ $\lambda_z^{\Psi_\Phi}(\etab)$ (see Eq.~\ref{eq:def_induced_scoring_rule_phi})} 
        \\
        \midrule
        Log~\citep{Good1952} 
        & $\eta_0\log\eta_0+\eta_1\log\eta_1$
        & $-\log \eta_z$
        \\
        \makecell[l]{Power ($\a\notin\{0,1\}$)\tablefootnote{If $\a=2$, 
        famously known as the Brier score~\citep{Brier1950,Gneiting--Raftery2007}.}} 
        & $\frac{\eta_0^\a+\eta_1^\a-1}{\a(\a-1)}$
        & \makecell[l]{$\frac{\eta_0^\a+\eta_1^\a}{\a}-\frac{\eta_z^{\a-1}}{\a-1}$} 
        \\
        Sym. inverse log 
        & $-\log \eta_0-\log\eta_1$
        & $\log\eta_0+\log\eta_1 + \frac{1}{\eta_z}$ 
        \\
        \makecell[l]{Pseudo-spherical $(\a\notin\{0,1\})$\\~~~\citep{Gneiting--Raftery2007}\tablefootnote{Called the spherical score when $\a=2$~\citep{Gneiting--Raftery2007}. When $\a\to1$, boils down to the log score. 
        }}
        & $\frac{1}{\a-1}(\frac{\eta_0^\a+\eta_1^\a}{2})^{\frac{1}{\a}}$ 
        & $-\frac{2^{-\frac{1}{\a}}}{\a-1}(\frac{\eta_z}{(\eta_0^\a+\eta_1^\a)^{\frac{1}{\a}}})^{\a-1}$
        \\
        \bottomrule
    \end{tabular}
\label{tab:binary_scoring_rules}
\end{table}

Now, by naturally extending the definition of the elementary generating functions for the binary scoring rules, we can derive their multi-ary counterparts as shown in Table~\ref{tab:sym_scoring_rules}. 
We note that the multi-ary asymmetric scoring rules, when considered with our binary density ratio estimation framework below, boil down to the ones induced by the binary scoring rules. Therefore, since the nontrivial examples are from extending the symmetric scoring rules, we omit the multiary version of asymmetric rules.
\begin{table}[ht]
    \centering
    \caption{Examples of \emph{symmetric} strictly proper $(M+1)$-ary scores. 
    }%
    \small
    \begin{tabular}{l l l m{3.5cm} l}
        \toprule
        Symmetric scoring rule & $\Phi(\etab)$ 
        & \makecell[l]{ $\lambda_z^{\Psi_\Phi}(\etab)$ (see Eq.~\ref{eq:def_induced_scoring_rule_phi})} 
        & Known as 
        \\
        \midrule
        Log
        & $\langle \etab, \log {\etab}\rangle$
        & $-\log \eta_z$
        \\
        \makecell[l]{Power ($\a\notin\{0,1\}$)} 
        & $\frac{\|\etab\|_{\a}^{\a}}{\a(\a-1)}$ or $\frac{\|\etab\|_{\a}^{\a}-1}{\a(\a-1)}$
        & \makecell[l]{$\frac{\|\etab\|_\a^\a}{\a}-\frac{\eta_z^{\a-1}}{\a-1}$} 
        & \makecell[l]{Tsallis scoring rule\\\citep{Dawid--Musio2014}}.
        \\
        Sym. inverse log 
        & $-\sum_{z=0}^M \log \eta_z$
        & $\sum_{z=0}^M\log\eta_z + \frac{1}{\eta_z}$ 
        \\
        \makecell[l]{Spherical $(\a\notin\{0,1\})$}
        & $\frac{(M+1)^{-\frac{1}{\a}}}{\a-1}\|\etab\|_{\a}$ 
        & $-\frac{(M+1)^{-\frac{1}{\a}}}{\a-1}(\frac{\eta_z}{\|\etab\|_{\a}})^{\a-1}$
        \\
        \bottomrule
    \end{tabular}
\label{tab:sym_scoring_rules}
\end{table}

\section{Details on Extensions with Proper Scoring Rules}
\label{app:sec:scoring_rules}

In this section, we provide technical materials deferred from Section~\ref{sec:extensions} on the extensions with proper scoring rules. 

\subsection{Alternative Characterization of Proper Scoring Rule}
\label{app:sec:alternative_characterization_proper_scoring_rules}
An alternative, yet equivalent representation of a proper scoring rule is based on a convex function $\Phi(\etab)$ over $\etab\in \Delta([0:M])$.
One can induce a convex function $\Psi(\rhob)$ from a convex function $\Phi(\etab)$ by the perspective transformation:
\begin{align*}
\Psi_{\Phi}(\rhob)
\defeq (1+\rho_1+\ldots+\rho_M)
\Phi\biggl(\frac{[1;\rhob]}{1+\rho_1+\ldots+\rho_M}\biggr).
\end{align*}
\begin{theorem}
\label{thm:bregman_phi}
Given a differentiable function $\Phi\suchthat\Delta([0:M])\to\Real$, 
\[
\lambdab^{\Psi_{\Phi}}(\etab)
=\Bigl(\langle \etab, \nabla_{\etab}\Phi(\etab)\rangle-\Phi(\etab)\Bigr)\boldsymbol{1} -\nabla_{\etab}\Phi(\etab).
\numberthis
\label{eq:def_induced_scoring_rule_phi}
\]
\end{theorem}
\begin{proof}%
First, we can write
\begin{align*}
\Lc_{K;\nu}^{\Psi_\Phi}(\etab_\th)-\Lc_{K;\nu}^{\Psi_\Phi}(\etabref)
&= 
\E_{p(x_{1:K})}\Bigl[
\langle \etabref(x_{1:K}),\lambdab^{\Psi_\Phi}(\etab_\th(x_{1:K}))\rangle
-\langle \etabref(x_{1:K}),\lambdab^{\Psi_\Phi}(\etabref(x_{1:K}))\rangle
\Bigr].
\end{align*}
It is easy to check, from the definition of the $\Psi_\Phi$-induced scoring rule in Eq.~\eqref{eq:def_induced_scoring_rule_phi}, 
\begin{align*}
\langle \etabref,\lambdab^{\Psi_\Phi}(\etab_\th)\rangle
&= 
-\Phi(\etab_\th) - \langle \nabla_{\etab}\Phi(\etab_\th),\etab^*-\etab_\th\rangle.
\end{align*}
In particular, 
\begin{align*}
\langle \etabref,\lambdab^{\Psi_\Phi}(\etabref)\rangle
&= -\Phi(\etabref).
\end{align*}
Hence, we have 
\begin{align*}
\langle \etabref,\lambdab^{\Psi_\Phi}(\etab_\th)\rangle
-\langle \etabref,\lambdab^{\Psi_\Phi}(\etabref)\rangle
&= B_{\Phi}(\etabref,\etab_\th).\qedhere
\end{align*}    
\end{proof}
See the proof of Theorem~\ref{thm:bregman} in Appendix~\ref{app:sec:proofs} for a comparison.

We remark that, for a (strictly) convex function $\Phi$, $\Psi$ is (strictly) convex since the perspective transformation preserves the convexity, and thus 
\begin{align*}
\langle \etab^*, \lambdab^{\Psi_{\Phi}}(\etab)\rangle 
\ge \langle \etab^*, \lambdab^{\Psi_{\Phi}}(\etab^*)\rangle 
= - \Phi(\etab^*).
\end{align*}
We note that the right hand side is the \emph{Bayes optimal risk}. 
In other words, a convex function $\Phi(\cdot)$ can characterize a proper scoring rule as its (negative) Bayes-optimal risk.

\begin{theorem}
For $\nu\ge 0$,
\begin{align*}
\Lc_{K;\nu}^{\Psi_\Phi}(\etab_\th)-\Lc_{K;\nu}^{\Psi_\Phi}(\etabref)
&= \E_{p(x_{1:K})}\Bigl[B_{\Phi}\bigl(\etab^*(x_{1:K}), \etab_\th(x_{1:K})\bigr)\Bigr],
\end{align*}
For $\nu\ge0$ and a convex function $\Phi$, we have
\begin{align*}
-\Lc_{K;\nu}^{\Psi_\Phi}(\etab_\th)
\le -\Lc_{K;\nu}^{\Psi_\Phi}(\etabref)
&= \E_{p(x_{1:K})}\Bigl[\Phi\bigl(\etab^*(x_{1:K})\bigr)\Bigr].
\end{align*}
If $\nu>0$, for a strictly convex function $\Phi$, the equality is achieved if and only if $r_\th(x)=\frac{\pdata(x)}{\pnoise(x)}$.
\end{theorem}

If $\nu=0$, \ie if there is no anchor class 0, we can only estimate the density ratio up to a multiplicative constant, as the original InfoNCE guarantees.

\subsection{On Implementation}
\label{app:sec:bregman_implementation}
Here, we 
We say that a scoring rule $\lambdab$ is \emph{$\{y_1,y_2\}$-invariant} for $y_1\neq y_2\in \Yc$ if $\lambda_{y_1}(\etab)=\lambda_{y_2}(\etab')$ and $\lambda_{y_2}(\etab)=\lambda_{y_1}(\etab')$ for any $\etab,\etab'$ such that $\eta_y=\eta_{y}'$ for $y\not\in\{y_1,y_2\}$ and $\eta_{y_1}=\eta_{y_2}'$ and $\eta_{y_2}=\eta_{y_1}'$.
\begin{proposition}
If the scoring rule $\lambdab$ is $\{z_1,z_2\}$-invariant for any $\{z_1,z_2\}\subseteq \{1,2,\ldots,K\}$, 
we have
\begin{align*}
\Lc_{K;\nu}^{\Psi}(\etab_\th) 
&= \frac{K}{K+\nu}\E_{\pdata(x_1)\pnoise(x_2)\cdots \pnoise(x_K)}[\lambda_{1}^\Psi(\etab_\th(x_{1:K}))] \\
&\qquad + \frac{\nu}{K+\nu}\E_{\pnoise(x_1)\pnoise(x_2)\cdots \pnoise(x_K)}[\lambda_{0}^\Psi(\etab_\th(x_{1:K}))].
\numberthis\label{eq:def_proposed}
\end{align*} 
\end{proposition}

\subsection{Examples of InfoNCE-anchor-Type DRE Objectives}
\label{app:sec:examples_dre_objs}

Recall the examples of proper scoring rules in Appendix~\ref{app:sec:examples_scores}.
In Table~\ref{tab:objs_asym}, we first list the canonical consistent DRE objectives derived by asymmetric scoring rules (see Table~\ref{tab:binary_scoring_rules}). As noted earlier, the tensorization of InfoNCE-anchor does not have any effect with asymmetric scoring rules, and the objectives boil down to the standard binary DRE objectives.

\begin{table}[ht]
    \centering
    \caption{Examples of consistent DRE objectives derived from \emph{asymmetric} scoring rules (first half of Table~\ref{tab:binary_scoring_rules}). 
    Note that these objectives induced by asymmetric scoring rules do not depend on $K$ and $\nu$.}%
    \small
    \begin{tabular}{l l m{5.9cm} l l}
        \toprule
        Asym. scoring rule 
        & \makecell[l]{$\Lc_{K;\nu}^{\Psi}(\etab_\th)$ (see Eq.~\eqref{eq:def_proposed})} 
        & Known as
        \\       
        \midrule
        \makecell[l]{Log} 
        & $\E_{\pdata}[-\log r_\th(x)]+\E_{\pnoise}[r_\th(x)]$
        & KLIEP~\citep{Sugiyama--Suzuki--Nakajima--Kashima--von-Bunau--Kawanabe2008kliep} in DRE.
        NWJ~\citep{Ngyuen--Wainwright--Jordan2010} in MI estimation.
        \\\midrule
        \makecell[l]{Power $(\a\notin(0,1))$} 
        & $\E_{\pdata}[\frac{r_\th(x)^{\a-1}}{1-\a}]
        +\E_{\pnoise}[\frac{r_\th(x)^\a}{\a}]$
        & Robust DRE~\citep{Sugiyama--Suzuki--Kanamori2012}, KLIEP~\citep{Sugiyama--Suzuki--Nakajima--Kashima--von-Bunau--Kawanabe2008kliep} when $\a\to1$, LSIF~\citep{Kanamori--Hido--Sugiyama2009} when $\a=2$ in DRE.
        \\\hdashline
        \makecell[l]{(when $\a=2$)} 
        & $-\E_{\pdata}[r_\th(x)]
        +\half\E_{\pnoise}[r_\th(x)^2]$
        & %
        In MI estimation/DRE, known as $\chi^2$ or DRF~\citep{Tsai--Zhao--Yamada--Morency--Salakhutdinov2020}.
        In rep. learning, H-score~\citep{Wang--Wu--Huang--Zheng--Xu--Zhang--Huang2019}, %
        ~spectral contrastive loss~\citep{HaoChen--Wei--Gaidon--Ma2021}, CCA~\citep{Chapman--Wells--Aguila2023}, LoRA loss~\citep{Ryu--Xu--Erol--Bu--Zheng--Wornell2024}. %
        \\\midrule        
        Inverse log 
        & $\E_{\pdata}[\frac{1}{r_\th(x)}]+\E_{\pnoise}[\log r_\th(x)]$
        & 
        \\
    \bottomrule
    \end{tabular}
    \label{tab:objs_asym}
\end{table}
As noted in the last column of the table, these binary DRE objectives have been extensively used and studied in the various literature on DRE, MI estimation, and representation learning. We mention in passing that a recent paper~\citep{Ryu--Shah--Wornell2025}, building on noise-contrastive estimation~\citep{Gutmann--Hyvarinen2012}, revealed a connection between these rules and the maximum likelihood estimation principle.

In Table~\ref{tab:objs_sym_anchor}, we list the InfoNCE-anchor-type objectives based on the symmetric scoring rules (see Table~\ref{tab:sym_scoring_rules}).
Table~\ref{tab:objs_sym_noanchor} lists the corresponding InfoNCE-type objectives (i.e., without anchor). 
We remark that the Spherical objective in the main text corresponds to the last row in Table~\ref{tab:objs_sym_anchor}.

\begin{table}[ht]
    \centering
    \caption{Examples of InfoNCE-anchor-type DRE objectives ($\nu>0$), 
    derived from \emph{symmetric} scoring rules (Table~\ref{tab:sym_scoring_rules}). Here, $\rhob_\th(x_{1:K})\defeq [\nu,\rho_\th(x_1),\ldots,\rho_\th(x_K)]\in\Real_+^{K+1}$ and $\rho_\th(x)\defeq \frac{r_\th(x)}{\nu}$.
    The objective in the first row corresponds to our proposal InfoNCE-anchor. When $K=1,\nu=1$, it is also known as JS~\citep{Poole--Ozair--van-den-oord-Alemi--Tucker2019} or NT-Logistics~\citep{Chen--Kornblith--Norouzi--Hinton2020}.
    }%
    \small

    \begin{tabular}{l l l l l}
    \toprule
        Sym. scoring rule 
        & \makecell[l]{$\Lc_{K;\nu}^{\Psi_\Phi}(\etab_\th)$ (see Eq.~\eqref{eq:def_proposed})}  
        & 
        \\
        \midrule
        Log
        & \makecell[l]{$\frac{K}{K+\nu}\E_{\pdata(x_{1:K})}[-\log\frac{\rho_{\th}(x_1)}{\|\rhob_\th(x_{1:K})\|_1}] 
        +\frac{1}{K+\nu}\E_{\pnoise(x_{1:K})}[-\log\frac{\nu}{\|\rhob_\th(x_{1:K})\|_1}]$}
        & 
        \\\midrule
        \makecell[l]{Power\\~~~($\a\notin\{0,1\}$)} 
        & \makecell[l]{$\frac{K}{K+\nu}\E_{\pdata(x_{1:K})}[\frac{1}{\a}(\frac{\|\rhob_\th(x_{1:K})\|_\a}{\|\rhob_\th(x_{1:K})\|_1})^\a + \frac{1}{1-\a} (\frac{\rho_{\th}(x_1)}{\|\rhob_\th(x_{1:K})\|_1})^{\a-1}]$\\
        $~~~+ \frac{\nu}{K+\nu} 
        \E_{\pnoise(x_{1:K})}[\frac{1}{\a}(\frac{\|\rhob_\th(x_{1:K})\|_\a}{\|\rhob_\th(x_{1:K})\|_1})^\a + \frac{1}{1-\a} (\frac{1}{\|\rhob_\th(x_{1:K})\|_1})^{\a-1}]$} 
        \\\midrule
        Sym. inverse log 
        & \makecell[l]{$\frac{K}{K+\nu}\E_{\pdata(x_{1:K})}[\frac{\|\log \rhob_\th(x_{1:K})\|}{\|\rhob_\th(x_{1:K})\|_1}
        +\frac{\|\rhob_\th(x_{1:K})\|_1}{\rho_\th(x_1)}
        ]
        $\\$~~~
        + \frac{\nu}{K+\nu} 
        \E_{\pnoise(x_{1:K})}[\frac{\log \prod\rhob_\th(x_{1:K})}{\|\rhob_\th(x_{1:K})\|_1}
        +\|\rhob_\th(x_{1:K})\|_1
        ]$} 
        \\\midrule
        \makecell[l]{Pseudo-spherical\\~~~$(\a\notin\{0,1\})$}%
        & \makecell[l]{$\frac{K}{K+\nu}\E_{\pdata(x_{1:K})}[
        (\frac{\rho_\th(x_1)}{\|\rhob_\th(x_{1:K})\|_\a})^{\a-1}
        ]
        + \frac{\nu}{K+\nu} 
        \E_{\pnoise(x_{1:K})}[
        (\frac{1}{\|\rhob_\th(x_{1:K})\|_\a})^{\a-1}
        ]$} 
        \\\hdashline
        ~~~($\a=2$) 
        & \makecell[l]{$\frac{K}{K+\nu}\E_{\pdata(x_{1:K})}[
        \frac{\rho_\th(x_1)}{\|\rhob_\th(x_{1:K})\|_2}
        ]
        + \frac{\nu}{K+\nu} 
        \E_{\pnoise(x_{1:K})}[
        \frac{1}{\|\rhob_\th(x_{1:K})\|_2}
        ]$} 
        \\
        \bottomrule
    \end{tabular}
\label{tab:objs_sym_anchor}
\end{table}

\begin{table}[ht]
    \centering
    \caption{Examples of InfoNCE-type DRE objectives, derived from \emph{symmetric} scoring rules (Table~\ref{tab:sym_scoring_rules}).}%
    \small

    \begin{tabular}{l l m{3cm} l l}
    \toprule
        Sym. scoring rule 
        & \makecell[l]{$\Lc_{K;0}^{\Psi_\Phi}(\etab_\th)$ (see Eq.~\eqref{eq:def_proposed})}  
        & Known as
        \\
        \midrule
        Log 
        & \makecell[l]{$\E_{\pdata(x_{1:K})}[-\log\frac{r_{\th}(x_1)}{\|\rb_\th(x_{1:K})\|_1}] 
        $}
        & InfoNCE~\citep{van-den-Oord--Li--Vinyals2018}/NT-Xent~\citep{Chen--Kornblith--Norouzi--Hinton2020}\\
        \midrule
        \makecell[l]{Power\\~~~($\a\notin\{0,1\}$)} 
        & \makecell[l]{$\E_{\pdata(x_{1:K})}[\frac{1}{\a}(\frac{\|\rb_\th(x_{1:K})\|_\a}{\|\rb_\th(x_{1:K})\|_1})^\a + \frac{1}{1-\a} (\frac{r_{\th}(x_1)}{\|\rb_\th(x_{1:K})\|_1})^{\a-1}]$} 
        \\
        \midrule
        Sym. inverse log 
        & \makecell[l]{$\E_{\pdata(x_{1:K})}[\frac{\log\prod \rb_\th(x_{1:K})}{\|\rb_\th(x_{1:K})\|_1}
        +\frac{\|\rb_\th(x_{1:K})\|_1}{r_\th(x_1)}
        ]$} 
        \\
        \midrule
        \makecell[l]{Pseudo-spherical\\~~~$(\a\notin\{0,1\})$}
        & \makecell[l]{$\E_{\pdata(x_{1:K})}[
        (\frac{r_\th(x_1)}{\|\rb_\th(x_{1:K})\|_\a})^{\a-1}
        ]$} 
        \\
        \hdashline
        ~~~($\a=2$) 
        & \makecell[l]{$\E_{\pdata(x_{1:K})}[
        \frac{r_\th(x_1)}{\|\rb_\th(x_{1:K})\|_2}
        ]
        $} 
        \\
        \bottomrule
    \end{tabular}
\label{tab:objs_sym_noanchor}
\end{table}

\section{Deferred Proofs}
\label{app:sec:proofs}

\subsection{Proof of Proposition~\ref{prop:cpc_loose}}
\label{app:proof:prop:cpc_loose}
\begin{proof}[Proof of Proposition~\ref{prop:cpc_loose}]
We have an alternative proof for a loose upper bound  
\[
-\Lc_{K;0}(\th)+\log K \le D(\pdata~\|~\pnoise).
\]
We first consider the NWJ variational lower bound of the KL divergence:
\[
D(\pdata~\|~\pnoise) \ge \E_\pdata[\log r]-\E_\pnoise[r] + 1.
\]
Here the equality holds if and only if $r(x)\equiv \frac{\pdata(x)}{\pnoise(x)}$.
For $K\ge 2$, consider two distributions $\pdata(x_1)\pnoise(x_2)\cdots \pnoise(x_K)$ and $\pnoise(x_1)\pnoise(x_2)\cdots \pnoise(x_K)$.
Applying the NWJ lower bound, we obtain 
\begin{align*}
&D(\pdata(x)~\|~\pnoise(x))\\
&= D(\pdata(x_1)\pnoise(x_2)\cdots \pnoise(x_K)~\|~\pnoise(x_1)\pnoise(x_2)\cdots \pnoise(x_K))\\
&\ge \E_{\pdata(x_1)\pnoise(x_2)\cdots \pnoise(x_K)}[\log r(x_1,\ldots,x_K)] - \E_{\pnoise(x_1)\pnoise(x_2)\cdots \pnoise(x_K)}[r(x_1,\ldots,x_K)] + 1.
\end{align*}
Note that, again, the equality is attained if and only if 
\[
r(x_1,\ldots,x_K)\equiv \frac{\pdata(x_1)}{\pnoise(x_1)}.
\]
Now, we consider a specific (suboptimal) parameterization of $r(x_1,\ldots,x_K)$ in the following form:
\[
r(x_1,\ldots,x_K)\gets \log\frac{r_\th(x_1)}{\frac{1}{K}\sum_{k=1}^K r_\th(x_k)}
\]
for some nonnegative-valued function $r_\th\suchthat \Xc\to\Real_{\ge 0}$. By symmetry, it is easy to show that
\[
\E_{\pnoise(x_1)\pnoise(x_2)\cdots \pnoise(x_K)}[r(x_1,\ldots,x_K)]=1.
\]
Hence, the NWJ lower bound simplifies to 
\begin{align}
D(\pdata(x)~\|~\pnoise(x)) \ge \E_{\pdata(x_1)\pnoise(x_2)\cdots \pnoise(x_K)}\Biggl[\log \frac{r_\th(x_1)}{\frac{1}{K}\sum_{k=1}^K r_\th(x_k)}\Biggr]
=-\Lc_{K;0}(\th),
\end{align}
which concludes the proof.
\end{proof}

\subsection{Proof of Theorem~\ref{thm:cpc}}

\begin{proof}[Proof of Theorem~\ref{thm:cpc}]
We start with the following upper bound
\begin{align*}
-\Lc_{K;\nu}(\th)
&=\E_{p(z)p(x_{1:K}|z)}[\log {p_\th(z|x_{1:K})}]\\
&\le\E_{p(z)p(x_{1:K}|z)}[\log {p(z|x_{1:K})}],
\end{align*}
where the upper bound is achieved when $p_\th(z|x_{1:K})=p(z|x_{1:K})$.
This is by the Gibbs inequality, or equivalently
\begin{align*}
\E_{p(x_{1:K})}\Bigl[D({p(z|x_{1:K})}~\|~{p_\th(z|x_{1:K})})\Bigr]\ge 0.
\end{align*}
We note that for $\nu=0$,
we have
\begin{align}
-\Lc_{K;0}(\th)+\log K
&\le\E_{p(z)p(x_{1:K}|z)}[\log {p(z|x_{1:K})}]+\log K\nonumber\\
&=\E_{\pdata(x_1)\pnoise(x_2)\cdots\pnoise(x_K)}\Biggl[\log\frac{\frac{\pdata(x_1)}{\pnoise(x_1)}}{\frac{1}{K}\sum_{z=1}^K \frac{\pdata(x_z)}{\pnoise(x_z)}}
\Biggr]\label{eq:kjsd}\\
&= D_{\mathsf{JS}}\Bigl(p(x_{1:K}|z=1),\ldots,p(x_{1:K}|z=K)\Bigr).\nonumber
\end{align}
The equality condition follows from the Gibbs inequality.
This proves the first inequality.

To prove the upper bound $\log K$,
continuing from Eq.~\eqref{eq:kjsd}, we have
\begin{align*}
-\Lc_{K;0}(\th)+\log K
&\le \E_{\pdata(x_1)\pnoise(x_2)\cdots\pnoise(x_K)}\Biggl[\log\frac{\frac{\pdata(x_1)}{\pnoise(x_1)}}{\frac{1}{K}\sum_{z=1}^K \frac{\pdata(x_z)}{\pnoise(x_z)}}
\Biggr]\\
&\le\E_{\pdata(x_1)\pnoise(x_2)\cdots\pnoise(x_K)}\Biggl[\log\frac{\frac{\pdata(x_1)}{\pnoise(x_1)}}{\frac{1}{K} \frac{\pdata(x_1)}{\pnoise(x_1)}}
\Biggr]=\log K.
\end{align*}

For the second upper bound, 
we apply Jensen's inequality with the concavity of the logarithmic function and obtain
\begin{align*}
\E_{\pdata(x_1)\pnoise(x_2)\cdots\pnoise(x_K)}\Biggl[\log{\frac{1}{K}\sum_{z=1}^K \frac{\pdata(x_z)}{\pnoise(x_z)}}
\Biggr]
&\ge \log\Biggl(\E_{\pdata(x_1)\pnoise(x_2)\cdots\pnoise(x_K)}\Biggl[{\frac{1}{K}\sum_{z=1}^K \frac{\pdata(x_z)}{\pnoise(x_z)}}
\Biggr]\Biggr)\\
&=
\log \Bigl(\frac{1}{K} \chi^2(\pdata~\|~\pnoise)+1\Bigr)\\
&\ge
\log \Bigl(\frac{1}{K} (e^{D(\pdata~\|~\pnoise)}-1)+1\Bigr).
\end{align*}
Here, $\chi^2(p~\|~q)\defeq \E_{p}[\frac{p}{q}]-1$ denotes the \emph{chi-squared divergence} between distributions $p$ and $q$.
The last inequality follows since $\chi^2(\pdata~\|~\pnoise)\ge e^{D(\pdata~\|~\pnoise)}-1$.
Rearranging the inequality proves the desired bound.
\end{proof}

\subsection{Proof of Proposition~\ref{prop:induced_loss}}

To prove this proposition, we need the following lemma. The definition of the generating function $\gv^{\lambdab}$ of a differentiable loss function $\lambdab$ is given in Eq.~\eqref{eq:def_generating_function} in Appendix~\ref{app:sec:scoring_rules_prelim}.
Recall that the definition of the induced loss function $\lambda^{\Psi}$ for a convex function $\Psi$ is in Eq.~\ref{eq:def_induced_scoring_rule}.
\begin{lemma}
\label{lem:gradient_induced_loss}
If $\Psi$ is twice differentiable,
$\gb^{\lambdab^\Psi}(\etab)\equiv 0$.
\end{lemma}
\begin{proof}[Proof of Proposition~\ref{prop:induced_loss}]
By Lemma~\ref{lem:gradient_induced_loss}, we have $\gb^{\lambdab^\Psi}(\etab)\equiv 0$.
Further, since $\etab\mapsto -f^{\lambdab^{\Psi}}(\etab)
= \eta_0 \Psi(\rhob)$ is a perspective of the function $\rhob\mapsto \Psi(\rhob)$, 
$f^{\lambdab^\Psi}$ must be (strictly) concave if $\Psi$ is (strictly) convex.
Hence, by Theorem~\ref{thm:proper_loss}, 
we conclude that $\lambdab^{\Psi}$ is (strictly) proper.
\end{proof}

We now prove 
Lemma~\ref{lem:gradient_induced_loss}.
\begin{proof}[Proof of Lemma~\ref{lem:gradient_induced_loss}]
Consider 
\[
g_z^{\lambdab^\Psi}(\etab) = \sum_{z=0}^M \eta_z \frac{\partial\lambda_z^{\Psi}(\etab)}{\partial\eta_z}.
\]
Note that $\rho_{z'}=\eta_{z'}/\eta_0$ for $z'=1,\ldots,M$, we have
\[
\frac{\partial\rho_{z'}}{\partial\eta_z} 
= \begin{cases}
-\frac{\rho_{z'}}{\eta_0} & z=0 \\
\frac{\ones\{z=z'\}}{\eta_0} & z=1,\ldots, M.
\end{cases}
\]
\paragraph{Case 1: $z=0$.}
If $z=0$, $\lambda_0^\Psi(\etab)=\langle\nabla \Psi(\rhob),\rhob\rangle - \Psi(\rhob)$. 
Hence, we have
\begin{align*}
\frac{\partial\lambda_0^{\Psi}(\etab)}{\partial\eta_0}
&= 
\sum_{z'=1}^M \frac{\partial\rho_{z'}}{\partial\eta_z} 
\frac{\partial}{\partial\rho_{z'}}(\langle\nabla \Psi(\rhob),\rhob\rangle - \Psi(\rhob))\\
&= \sum_{z'=1}^M
-\frac{\rho_{z'}}{\eta_0}
(\nabla^2\Psi(\rhob)\rhob)_{z'}\\
&= -\frac{1}{\eta_0}(\langle \rhob,\nabla^2\Psi(\rhob)\rhob)\rangle - (\nabla^2\Psi(\rhob)\rhob)_0).
\numberthis\label{eq:yzero_zzero}
\end{align*}
If $1\le z\le M$, $\lambda_z^\Psi(\etab)=-(\nabla \Psi(\rhob))_z$, and thus
\begin{align*}
\frac{\partial\lambda_z^{\Psi}(\etab)}{\partial\eta_0}
&= \sum_{z'=1}^M \frac{\partial\rho_{z'}}{\partial\eta_z} 
\frac{\partial}{\partial\rho_{z'}}\Bigl(-\frac{\partial\Psi(\rhob)}{\partial\rho_z}\Bigr)\\
&= \frac{1}{\eta_0} \sum_{z'=1}^M \rho_{z'} \frac{\partial^2\Psi(\rhob)}{\partial\rho_z \partial\rho_{z'}}\\
&= \frac{1}{\eta_0} (\nabla^2\Psi(\rhob)\rhob)_z.
\numberthis\label{eq:yzero_znonzero}
\end{align*}
From \eqref{eq:yzero_zzero} and \eqref{eq:yzero_znonzero}, we have
\begin{align*}
g_0^{\lambdab^\Psi}(\etab) 
&= \sum_{z=0}^M \eta_z \frac{\partial\lambda_z^{\Psi}(\etab)}{\partial\eta_0}\\
&= -(\langle \rhob,\nabla^2\Psi(\rhob)\rhob)\rangle - (\nabla^2\Psi(\rhob)\rhob)_0)
+ \sum_{z=1}^M \frac{\eta_z}{\eta_0} (\nabla^2\Psi(\rhob)\rhob)_z\\
&= -\langle \rhob,\nabla^2\Psi(\rhob)\rhob)\rangle 
+ \sum_{z=0}^M \rho_z (\nabla^2\Psi(\rhob)\rhob)_z\\
&= 0.
\end{align*}

\paragraph{Case 2: $1\le z\le M$.}
If $z=0$, we have 
\[
\frac{\partial\lambda_0^{\Psi}(\etab)}{\partial\eta_z}
= \frac{1}{\eta_0} (\nabla^2\Psi(\rhob)\rhob)_z.
\numberthis\label{eq:ynonzero_zzero}
\]
If $1\le z\le M$, 
\begin{align*}
\frac{\partial\lambda_z^{\Psi}(\etab)}{\partial\eta_z}
&= -\frac{\partial}{\partial\eta_z} \frac{\partial\Psi(\rhob)}{\partial\rho_z} \\
&= -\frac{\partial\rho_z}{\partial\eta_z} \frac{\partial^2\Psi(\rhob)}{\partial\rho_z \partial\rho_z} \\
&= -\frac{1}{\eta_0} \frac{\partial^2\Psi(\rhob)}{\partial\rho_z \partial\rho_z}.
\numberthis\label{eq:ynonzero_znonzero}
\end{align*}
Therefore, from \eqref{eq:ynonzero_zzero} and \eqref{eq:ynonzero_znonzero}, we have
\begin{align*}
g_z^{\lambdab^\Psi}(\etab) 
&= \sum_{z=0}^M \eta_z \frac{\partial\lambda_z^{\Psi}(\etab)}{\partial\eta_z}\\
&= (\nabla^2\Psi(\rhob)\rhob)_z - \sum_{z=1}^M \rho_z \frac{\partial^2\Psi(\rhob)}{\partial\rho_z \partial\rho_z}\\
&= (\nabla^2\Psi(\rhob)\rhob)_z - (\nabla^2\Psi(\rhob)\rhob)_z\\
&= 0.
\end{align*}
Hence, we conclude that $\gb^{\lambdab^\Psi}(\etab)\equiv 0$.
\end{proof}

\subsection{Proof of Theorem~\ref{thm:bregman}}
We note that, while the following proof is self-contained, a more detailed technical discussion on the general relationship between proper scoring rule and Bregman divergence minimization in Appendix~\ref{app:sec:bregman_connection}.
\begin{proof}[Proof of Theorem~\ref{thm:bregman}]
Note that we can write
\begin{align*}
\Lc_{K;\nu}^{\Psi}(\etab_\th)-\Lc_{K;\nu}^{\Psi}(\etabref)
&= \E_{p(x_{1:K})}\Bigl[
\langle \etabref(x_{1:K}),\lambdab^{\Psi}(\etab_\th(x_{1:K}))\rangle
-\langle \etabref(x_{1:K}),\lambdab^{\Psi}(\etabref(x_{1:K}))\rangle
\Bigr].
\end{align*}
Now, it is easy to check that, we have
\begin{align*}
\langle \etabref,\lambdab^{\Psi}(\etab_\th)\rangle
&= \etaref_0\Bigl(
-\Psi(\rhob_\th) - \langle \nabla_{\rhob}\Psi(\rhob_\th),\rhob^*-\rhob_\th\rangle\Bigr).
\end{align*}
In particular, 
\begin{align*}
\langle \etabref,\lambdab^{\Psi}(\etabref)\rangle
&= -\etaref_0\Psi(\rhobref).
\end{align*}
Hence, we have
\begin{align*}
\langle \etabref,\lambdab^{\Psi}(\etab_\th)\rangle
-\langle \etabref,\lambdab^{\Psi}(\etabref)\rangle
&= \etaref_0B_{\Psi}(\rhobref,\rhob_\th).
\end{align*}
From this expression, we have
\begin{align*}
\Lc_{K;\nu}^{\Psi}(\etab_\th)-\Lc_{K;\nu}^{\Psi}(\etabref)
&= \E_{p(x_{1:K})}\Bigl[\etaref_0B_{\Psi}(\rhobref,\rhob_\th)\Bigr]\\
&= \E_{p(x_{1:K})}\Bigl[p(z=0|x_{1:K})B_{\Psi}(\rhobref,\rhob_\th)\Bigr]\\
&= p(z=0)\E_{p(x_{1:K}|z=0)}\Bigl[B_{\Psi}(\rhobref,\rhob_\th)\Bigr].
\end{align*}
Since $p(z=0)=\frac{\nu}{K+\nu}$ and $p(x_{1:K}|z=0)=\pnoise(x_1)\pnoise(x_2)\cdots \pnoise(x_K)$ by definition, this concludes the proof.
\end{proof}

\section{Experiment Details}
\label{app:sec:exp_details}
This section provides the details on the experiments in the main text. 
All implementations are based on PyTorch and all experiments were performed on a single NVIDIA GeForce RTX 3090.

\subsection{MI Estimation}
We conducted a series of mutual information (MI) estimation experiments across three distinct data modalities: synthetic Gaussian variables, image-based representations from MNIST, and text embeddings derived from the IMDB dataset, using the standardized mibenchmark framework~\citep{Lee--Rhee2024}. Each experiment paired a 10-dimensional synthetic source variable $X \in \mathbb{R}^{10}$ with a modality-specific target variable $Y$, varying in dimensionality depending on the data type. Across all experiments, we used a consistent training configuration: models were optimized using Adam with a learning rate of 1e-4, trained in stepwise mode for 20,000 iterations.

Across all setups, we evaluated a fixed set of mutual information estimators, including NWJ, NWJ-Plugin, JS, JS-Plugin, InfoNCE, InfoNCE-Anchor, Density Ratio Fitting, and Spherical, with both joint and separable critic types. The critic network in all cases was an MLP composed of two hidden layers with 512 units, ReLU activations, and no normalization or dropout layers. Critic architectures projected inputs into a shared 16-dimensional embedding space. For joint critics, $X$ and $Y$ pairs were concatenated and passed through a single encoder, whereas for separable critics, independent encoders $g(x)$ and $h(y)$ were used. 
Batch size varied by dataset: 16 for Gaussian data and 64 for MNIST and IMDB.
We refer the readers to \citep{Lee--Rhee2024} and their codebase for the rest of the details including the data generation mechanism.

Figure~\ref{fig:mi_estimation_gaussian_cubic} summarizes the result of MI estimation for the Gaussian experiment with cubic transformation with varying batch sizes.
It clearly shows that InfoNCE-anchor exhibits a consistent performance, but we note that $\mathsf{JS}_{\mathsf{plugin}}$ also performs remarkably well in this simple benchmark.

\begin{figure}[ht]
    \centering
    \begin{subfigure}{\textwidth}
        \centering
        \includegraphics[width=\linewidth]{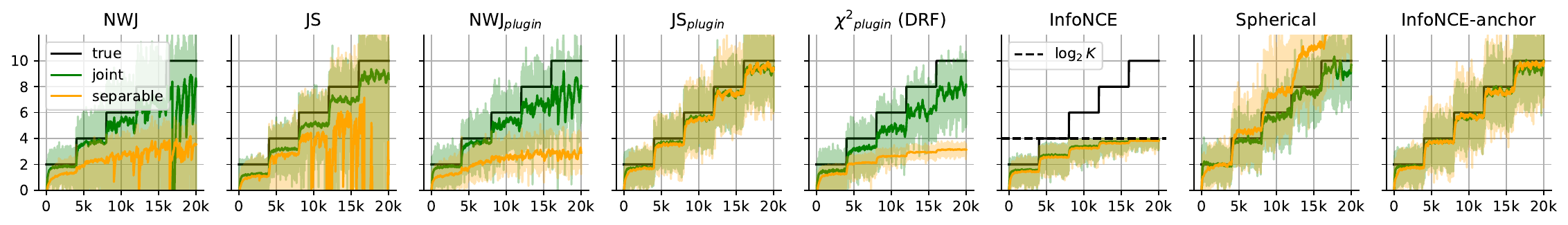}\vspace{-.25em}
        \caption{Gaussian (cubic) with batch size 16.}
    \end{subfigure}
    
    \begin{subfigure}{\textwidth}
        \centering
        \includegraphics[width=\linewidth]{figs/mi_gaussian_cubic_bs64.pdf}\vspace{-.25em}
        \caption{Gaussian (cubic) with batch size 64.}
    \end{subfigure}
    
    \begin{subfigure}{\textwidth}
        \centering
        \includegraphics[width=\linewidth]{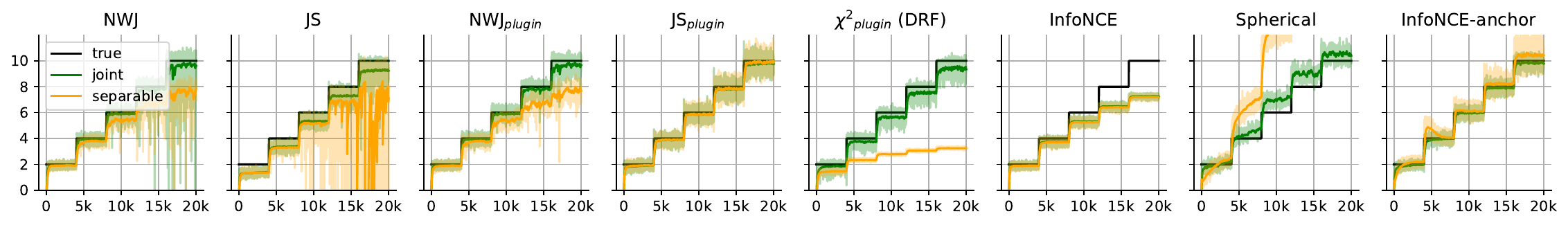}\vspace{-.25em}
        \caption{Gaussian (cubic) with batch size 256.}
    \end{subfigure}

    \caption{Summary of MI estimation results on the standard benchmark on the Gaussian cubic data, with different batch sizes.}
    \label{fig:mi_estimation_gaussian_cubic}
\end{figure}

\subsection{Protein Interaction Prediction}

We followed the same setup of \citet{Gowri--Lun--Klein--Yin2024}, and here we briefly overview the essential part.
We conducted experiments on two datasets derived from ProtTrans5-encoded protein embeddings: one composed of 22,229 kinase–target pairs and another with 1,702 ligand–receptor pairs. Each protein is represented by a 1,024-dimensional vector, and all embeddings were whitened and clipped to the range $[-10, 10]$. Across 20 trials, 170 proteins were randomly selected and held out per trial, ensuring that no interaction in the training set included any of the held-out proteins. The remaining interactions were used for training a mutual information estimator.

Our approach trains a separable critic network to estimate the density ratio via the InfoNCE-anchor objective. 
The critic architecture is a MLP with 4 hidden layers, each containing 256 units, and outputs 32-dimensional embeddings for each input protein, separate encoders  $f(x)$ and $g(y)$ for each side of the pair. 
ReLU activation was used, and no normalization layers were applied by default. We used the Adam optimizer with a learning rate of 1e-4, batch size of 64, and 10,000 training steps. 
We implement early stopping with a patience of 500 steps, based on validation loss, which is monitored every 500 iterations. 
The final model is selected based on the best validation performance and is then used to estimate pointwise mutual information (PMI) for held-out protein pairs. 

We present ROC curves (Figure~\ref{fig:roc_each_problem}) and histograms of learned PMI values (Figure~\ref{fig:histograms}) for each estimator. These two figures clearly demonstrate that InfoNCE-anchor exhibit the best discriminative power.

\begin{figure}[ht]
    \centering
    \begin{subfigure}{\textwidth}
        \centering
        \includegraphics[width=\textwidth]{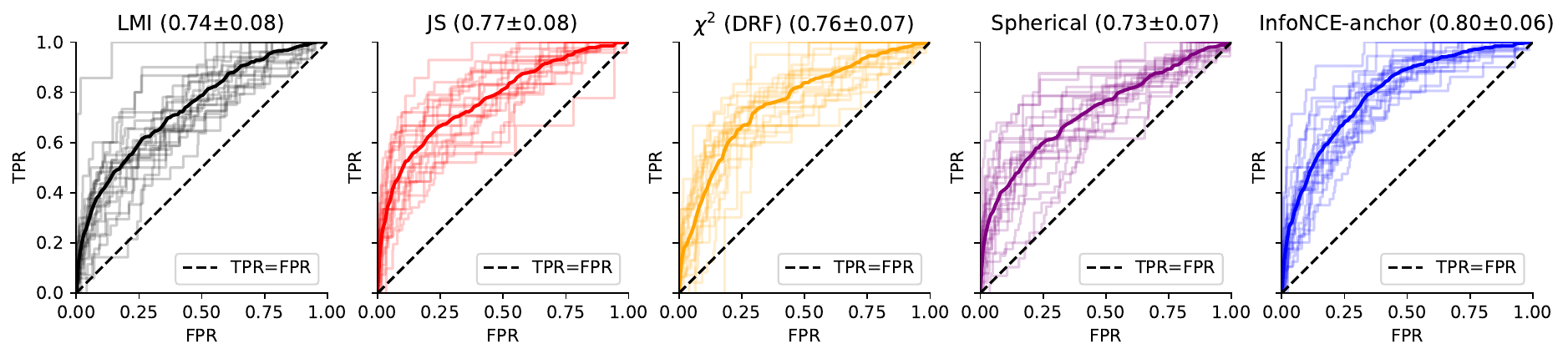}\vspace{-.25em}
        \caption{Kinase}
    \end{subfigure}
    
    \begin{subfigure}{\textwidth}
        \centering
        \includegraphics[width=\textwidth]{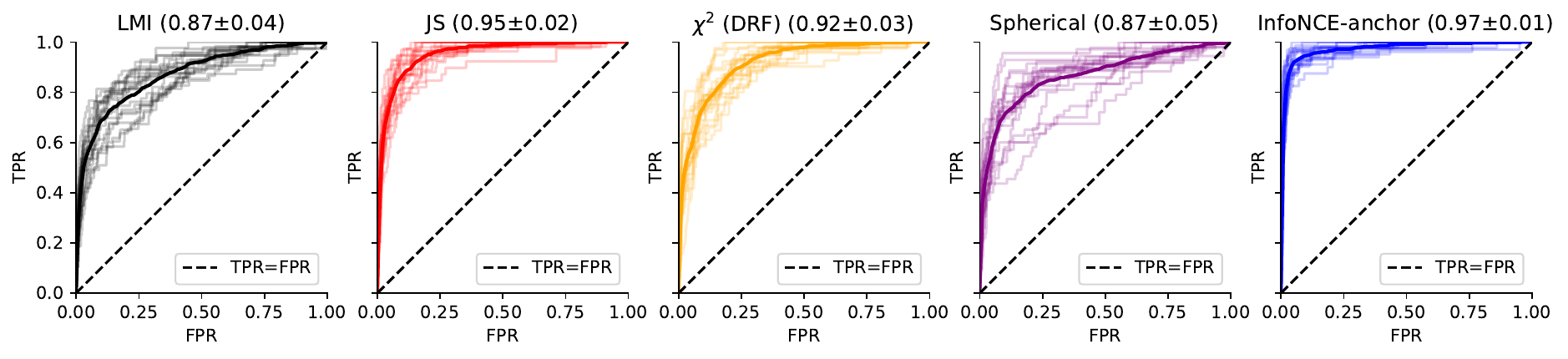}\vspace{-.25em}
        \caption{Ligand}
    \end{subfigure}
    
    \caption{ROC curves from different estimators.}
    \label{fig:roc_each_problem}
\end{figure}

\begin{figure}[ht]
    \centering
    \begin{subfigure}{\textwidth}
        \centering
        \includegraphics[width=\textwidth]{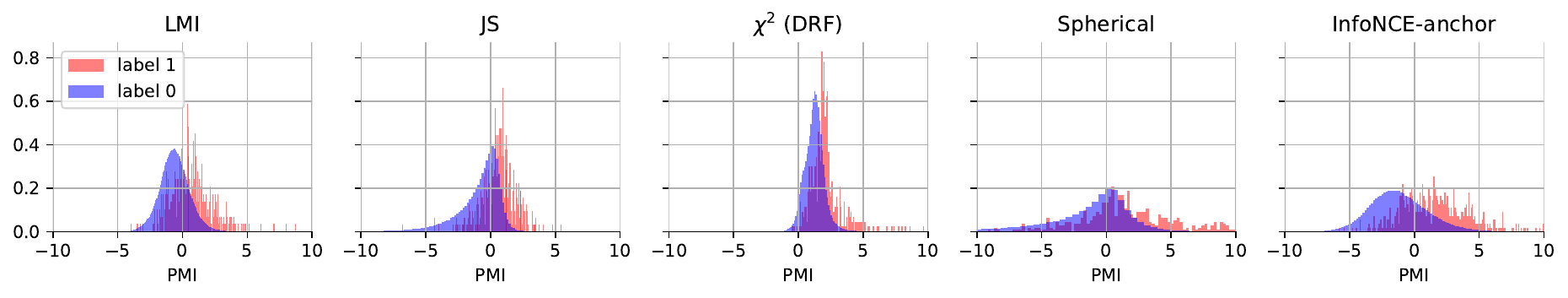}\vspace{-.25em}
        \caption{Kinase--target pair prediction.}
    \end{subfigure}
    
    \begin{subfigure}{\textwidth}
        \centering
        \includegraphics[width=\textwidth]{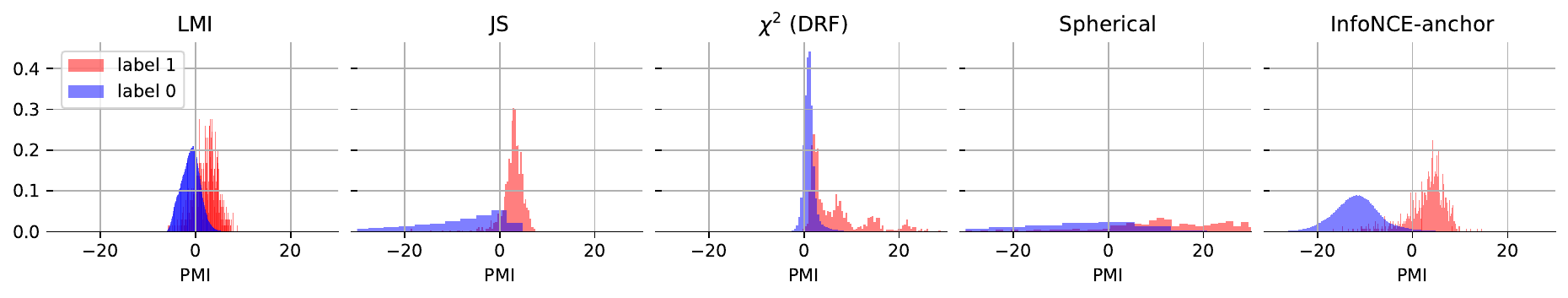}\vspace{-.25em}
        \caption{Ligand--receptor pair prediction.}
    \end{subfigure}
    
    \caption{Histograms of pointwise MI ($\log\frac{p(x,y)}{p(x)p(y)}$) from different estimators.}
    \label{fig:histograms}
\end{figure}

\subsection{Self-supervised Representation Learning}
Here, we provide details on the objective functions we considered in the experiment.
We used the temperature parameter $\tau=0.2$ throughout, unless stated otherwise. 

\begin{itemize}
\item InfoNCE: Log score, $K=B-1$, $\nu=0$, PMI factorization.
\item InfoNCE-anchor: Log score, $K=B-1$, $\nu=1$, PMI factorization.
\item JS: Log score, $K=1$, $\nu=1$, PMI factorization.
\item Spherical: Spherical score, $K=B-1$, $\nu=1$, PD factorization.
\item $\chi^2$: Asymmetric power score with $\a=2$, $K=1$, $\nu=1$. In this case, $\tau=0.1$ was used.
\end{itemize}

We found that the PMI factorization was not effective for all scoring rules other than the log score.
The rest of the experimental details can be found from the codebase of \citet{JMLR:v23:21-1155}.

\end{document}